\documentclass[12pt, final]{l4dc2023}
\usepackage{microtype}
\usepackage{booktabs}
\usepackage{balance}
\usepackage{enumitem}
\usepackage{hyperref}
\usepackage{url}
\usepackage{wrapfig}
\usepackage[format=plain]{caption}
\usepackage{bbold}

\RestyleAlgo{ruled}
\usepackage{floatrow}
\newcommand{\T}[1]{{\mathcal{#1}}} 
\usepackage{hyperref}
\usepackage{wrapfig}
\usepackage{amsfonts,bm} 

\def\eqref#1{equation~\ref{#1}}
\def\Eqref#1{Equation~\ref{#1}}

\def\1{\bm{1}}

\newcommand{\plusminus}[1] {\scriptsize{$\pm$ #1}}

\DeclareMathAlphabet{\mathsfit}{\encodingdefault}{\sfdefault}{m}{sl}
\SetMathAlphabet{\mathsfit}{bold}{\encodingdefault}{\sfdefault}{bx}{n}

\newtheorem*{prop*}{Proposition}

\newcommand{\SOtwo}{\mathrm{SO}(2)}

\title[Probabilistic Symmetry for Multi-Agent Dynamics]{Probabilistic Symmetry for Multi-Agent Dynamics}
\usepackage{times}

\author{%
 \Name{Sophia Sun} \Email{sophiasun@eng.ucsd.com}\\
 \addr University of California, San Diego
 \AND
 \Name{Robin Walters} \Email{r.walters@northeastern.edu}\\
 \addr Northeastern University
  \AND
 \Name{Jinxi Li} \Email{jinxi.li@connect.polyu.hk}\\
 \addr Hong Kong Polytechnic University
  \AND
 \Name{Rose Yu} \Email{roseyu@ucsd.edu}\\
 \addr University of California, San Diego%
}

\begin{document}

\maketitle

\begin{abstract}%
Learning multi-agent dynamics is a core AI problem with broad applications in robotics and autonomous driving. While most existing works focus on deterministic prediction, producing probabilistic forecasts to quantify uncertainty is critical for downstream decision-making tasks such as motion planning and collision avoidance. 
By leveraging the internal symmetry in multi-agent dynamics, specifically rotational equivariance, we can improve not only the accuracy, but also calibration of our probabilistic forecasts. We propose a novel deep dynamics model, Probabilistic Equivariant Continuous COnvolution (PECCO) for probabilistic prediction of multi-agent trajectories. PECCO extends equivariant continuous convolution to model the joint velocity distribution of multiple agents. It uses dynamics integration to propagate the uncertainty from velocity to position. We introduce Energy Score, a proper scoring rule, to evaluate probabilistic predictions. On both synthetic and real-world datasets, PECCO shows significant improvements in accuracy and calibration compared to non-equivariant baselines. 

\noindent Our code is released at \url{https://github.com/Rose-STL-Lab/PECCO}. The appendix of the paper can be accessed at  \url{https://arxiv.org/abs/2205.01927}.
\end{abstract}

\begin{keywords}%
Multi-Agent Modeling, Probabilistic forecasting, deep dynamics model, uncertainty quantification, equivariant neural networks
\end{keywords}

\section{Introduction}
Predicting the future trajectory of multiple agents is a critical task with applications in autonomous driving \citep{argoverse}, social behavioral modeling \citep{sun2021multi}. 
In practice, the problem is difficult due to the inherent stochasticity of human motion, and the strong inter-dependency among the agents where the number of interactions grows quadratically with the number of agents. Moreover, agent movements are often influenced by environmental features such as road markings, boundaries, and social preference, which are impossible to measure and model effectively. Such a partially observed setting introduces a significant amount of uncertainty. 

Many recent works on learning multi-agent dynamics has shifted to probabilistic modeling as a principled framework to represent uncertainty \citep{mfp, trajectron++}. However, common metrics used in probabilistic trajectory predictions works, such as minimum average displacement of 6 samples (minADE) , do not fully reflect the quality of probabilistic forecasts. A probabilistic prediction should be \textit{calibrated} and \textit{sharp}; that is, the predicted distribution must cover likely future scenarios without being so broad and uncertain as to be useless.

\begin{wrapfigure}{r}{0.5\textwidth}
\centering
\subfigure[Original scene]{\label{fig:noneq0}
  \includegraphics[width=.45\textwidth]{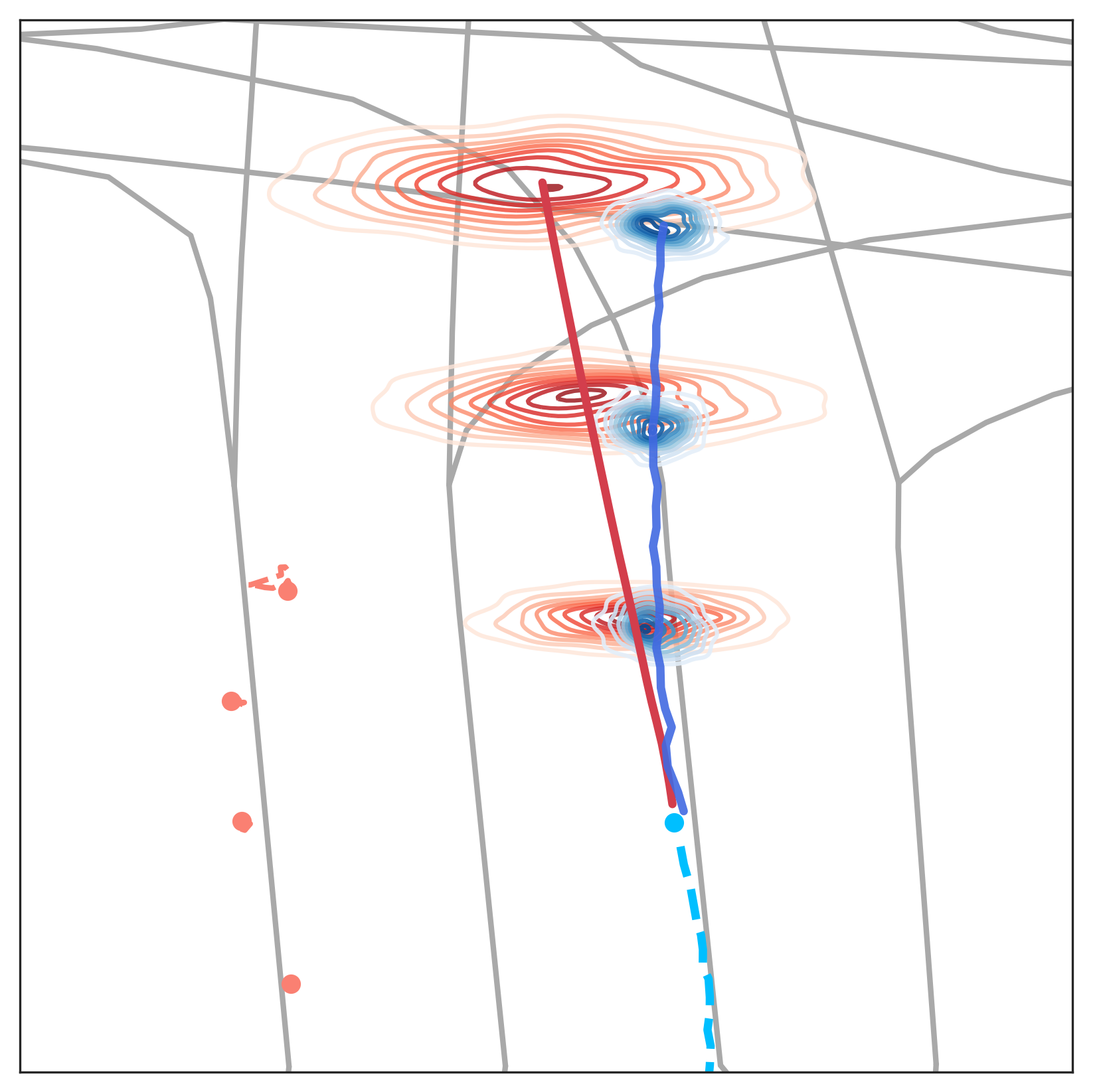}}
\subfigure[Rotated scene]{\label{fig:eq0}
  \includegraphics[width=.45\textwidth]{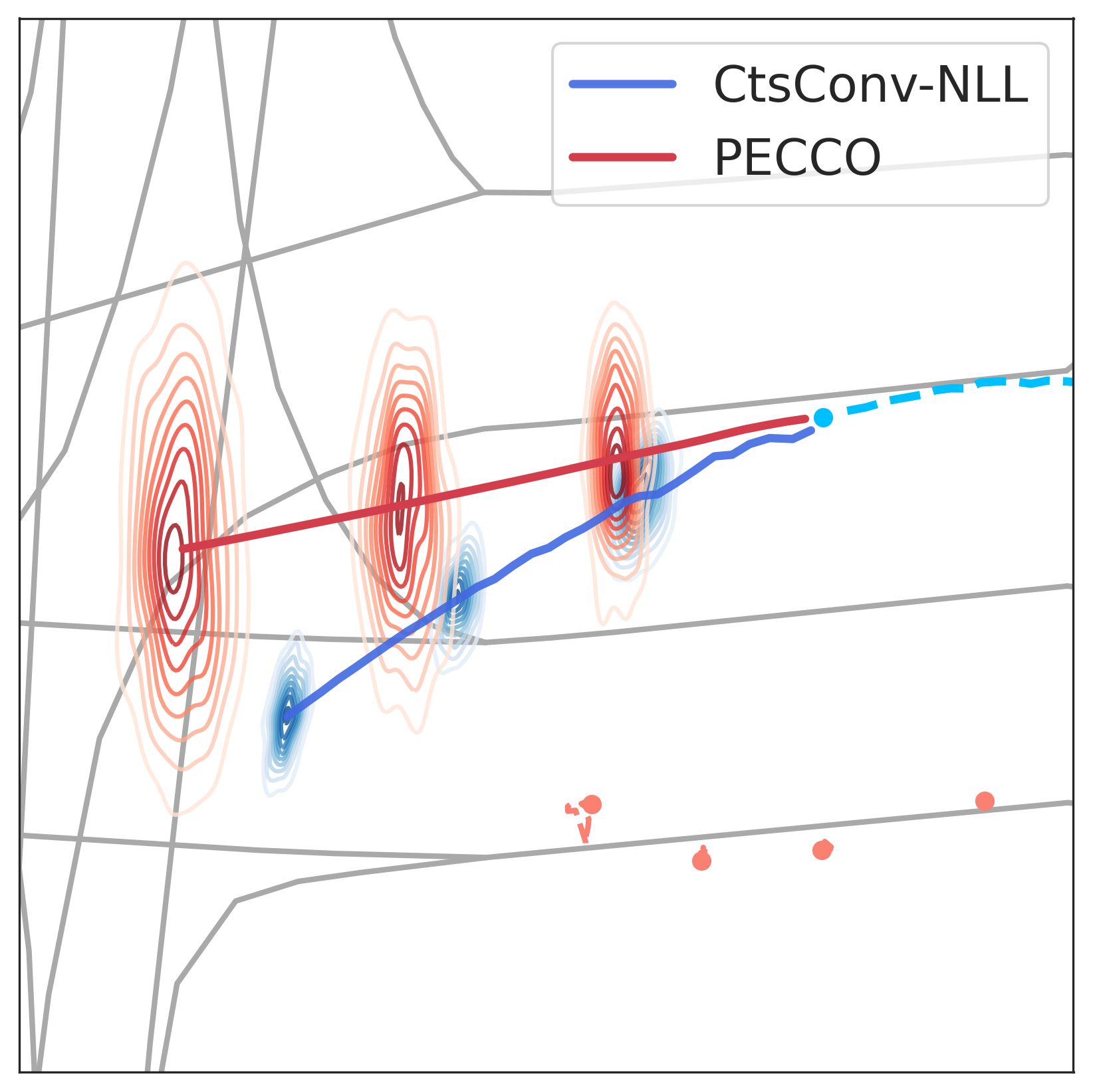}}
\caption{Prediction on the same scene rotated by 90 degrees. PECCO is consistent in trajectory and uncertainty prediction, whereas the non-equivariant  model (CtsConv) fails.}
\label{fig:rotpred}
\end{wrapfigure}

In this paper, we propose a \textbf{P}robabilistic \textbf{E}quivariant \textbf{C}ontinuous \textbf{CO}nvolutional model (PECCO). PECCO is an equivariant probabilistic trajectory prediction model. Our key insight is to exploit \textit{symmetry} to estimate multidimensional conditional distributions with limited data.  We assume the predicted probability distribution is rotation and translation equivariant. That is, if the input data is transformed, the probability distribution will also be likewise transformed. In Figure \ref{fig:rotpred}, we see the same car approaching an intersection from either the south or east. The scenes are related by a $\pi/2$ rotation.  As the absolute compass direction is not particularly meaningful for local trajectory prediction, the model should thus output the same probability distributions over future trajectories for the car coming from the east as that coming from the north, but rotated by $\pi/2$.  Rotational equivariance not only allows our model to produce physically consistent predictions, the multiplicative nature of equivariance also allows us to model a probability space with a smaller sample size \citep{JMLR:v21:19-322}. For each sample which an equivariant model is trained on, an equivariant model learns as if it were trained on all transformations of that sample by the symmetry group \citep{wang2020incorporating}.

PECCO also mitigates issues posed by other methods for enforcing equivariance such as data augmentation and normalization. Data augmentation adds rotated versions of data samples to the training dataset such that the model learns rotational equivariance. However, this slows training drastically, requires greater model capacity, and rarely achieves the level of equivariance or accuracy as equivariant neural networks \citep{trajectron++}. Data normalization is a technique that rotates the scene to the agent's reference frame for each data sample, as in \cite{vectornet}.  However, in the multi-agent setting, it is impossible to rotate the scene for multiple agents without a canonical reference frame. PECCO allows the weights to be relative to the local orientation of each agent without the need to rotate the scene repeatedly.

Our main contributions are two folds: (1) We design an equivariant neural network, PECCO, for probabilistic forecast of multi-agent dynamics, and (2) we demonstrate that by incorporating symmetry, PECCO improves both calibration and sharpness of probabilistic forecasts on a synthetic particle dataset and two real-world benchmark datasets.

\section{Related Work}
\paragraph{Trajectory Prediction.} 
%
Multi-agent trajectory forecasting has been extensively studied, 
approaches ranged from Kalman filters \citep{kalman1960} to non-linear Gaussian Process Regression models \citep{gaussianprocess}. However, these methods either rely on strong assumptions of the dynamics, or do not explicitly model multi-agent interactions. We refer readers to \cite{trajpredsurvey} for a comprehensive survey of such methods. 
Advancements in deep learning have allowed  flexible modeling of trajectory dynamics \citep{sociallstm, ctsconv, convsocialpooling, sophie, laneGCN, ecco, vectornet, simplex}, but they focus mainly on point estimation without uncertainty.



Recent methods have shifted to predicting distributions of  future trajectories, capturing uncertainty in dynamics. There are two main categories for probabilistic forecasting: (1) \textit{explicitly} via exact likelihood \citep{mfp, multipath,densetnt} and variational inference \citep{CVAE, trajectron++, desire}, or (2) \textit{implicitly} with Generative Adversarial Networks (GANs) \citep{socialgan,liu2019naomi}. Our work falls into the first category where we model the distributions parametrically. Parametric models allow us to evaluate the likelihood of future trajectories, which are useful for downstream planning tasks \citep{multipath, schwarting2018planning}.

Despite the development in probabilistic modeling, there is no standard metric for quantifying uncertainty of the prediction. Negative log likelihood often overfits the distribution \citep{guo2017calibration}, and best-of-n-sample results do not evaluate the full distribution \citep{rethinking}. We argue that probabilistic forecasts should accurately reflect the uncertainty in the model predictions. We propose using proper scoring rules such as Energy Score or mean interval score \citep{properscoring} for evaluating probabilistic forecasts. 


\paragraph{Equivariant Deep Learning.}
Geometric deep learning that leverages invariance and symmetries has found wide applications in areas ranging from image recognition \citep{bao2019equivariant, worrall2017harmonic, weiler2019e2cnn} to reinforcement learning \citep{equimdp}. Equivariant neural networks are studied for modeling dynamics as well - \cite{fuchs2020se} use $\mathrm{SE}(3)$-equivariant transformers to predict trajectories for a small number of particles as a regression task, and \cite{ecco} proposed a $\mathrm{S0}(2)$ equivariant continuous convolution for traffic trajectory prediction. All the methods mentioned above are deterministic. \cite{kohler2020equivariant} and \cite{ennormalizingflows} studies equivariant normalizing flows for modeling symmetric densities, however their domains focus on generative modeling and therefore differ from our work significantly. To our knowledge, no previous work has studied equivariant neural networks for probabilistic dynamics forecasting. 

\paragraph{Uncertainty Quantification (UQ).}
Uncertainty quantification is critical for risk assessment in safety-critical domains. Properly quantified uncertainties can be used to create probabilistic constraints and generate more robust planning and control strategies \citep{ostafew2016learning, bujarbaruah2019adaptive}. With the increasing use of deep learning in forecasting tasks, many works have UQ for neural networks \citep{saftyassurance, wu2021quantifying, guo2017calibration}. \cite{2021conformal} proposes a conformal prediction algorithm for 1D RNN forecasters with a prediction region with coverage guarantees. However, these works focus only on classification or 1-dimensional forecasts. \cite{salinas2019high, salinas2020deepar} use autoregressive RNNs for probabilistic forecasting of multiple time series, however, their method cannot explicitly model spatial relations. We present a model design for multi-agent dynamics and produce probabilistic distributions with better calibration.

\section{Background}
We give a short background on  symmetry and equivariance and their probabilistic extension.
\paragraph{Symmetry and Equivariance.}\label{subsec:symmetry}
A \emph{symmetry group} $G$  is a set together with a composition operation $\circ \colon G \times G \to G$ which is associative and has an identity and inverses. The group $G$ can transform a vector space $V$ by specifying a \emph{representation} which is a mapping $\rho \colon G \to \mathrm{GL}_n(V)$ sending each element of the group $G$ to an invertible $n \times n$ matrix such that $\rho(g_1 \circ g_2) = \rho(g_1) \rho(g_2)$.  

Given a function $f \colon X \to Y$ such that $G$ has representations $\rho_X$ and $\rho_Y$ acting on $X$ and $Y$ respectively, we say $f$ is $G$-\emph{equivariant} if for all $x \in X$ and $g \in G$, we have $\rho_Y(g)f(x) = f(\rho_X(g) x )$.  That is, a transformation of the input of $x$ induces a corresponding transformation of the output.  \emph{Invariance} for the function $f$ is a special case in which $\rho_Y(g)y = y$. 



\paragraph{\texorpdfstring{$\mathbf{SO(2)}$}{SO(2)} Equivariant Continuous Convolution.}
Continuous convolution \citep{ummenhofer2019lagrangian} generalizes discrete convolution. The feature vector $\mathbf{f}^{(i)} \in \mathbb{R}^{c_\mathrm{in}}$ of particle $i$ forms a vector field $\mathbf{f}$, and the kernel of the convolution $K: \mathbb{R}^2 \mapsto \mathbb{R}^{c_\mathrm{out}\times c_\mathrm{in}}$ forms a matrix field: for each point $\mathbf{x} \in \mathbb{R}^2$, $K(\mathbf{x})$ is a $c_\mathrm{out}\times c_\mathrm{in}$ matrix. The continuous convolution is then defined by 
\[
 \mathbf{g}^{(i)} = \sum_j a(\|\mathbf{x}^{(j)} - \mathbf{x}^{(i)}\|) K(\mathbf{x}^{(j)} - \mathbf{x}^{(i)}) \cdot \mathbf{f}^{(j)}.
\]
By \cite{weiler2019e2cnn}, this is $\SOtwo$-equivariant if 
$
K(gv) = \rho_\mathrm{out}(g) K(v) \rho_\mathrm{in}(g^{-1})$.

ECCO \citep{ecco} defines the convolution kernel $K$ in polar coordinates $K(\theta, r)$. Let $\mathbb{R}^{c_\mathrm{in}}$ and $\mathbb{R}^{c_\mathrm{out}}$ be $\SOtwo$-representations $\rho_\mathrm{in}$ and $\rho_\mathrm{out}$ respectively, then the convolution kernel satisfies the equivariance condition as follows, making the continuous convolution $\SOtwo$-equivariant.
\[
 K(\theta + \phi, r) = \rho_\mathrm{out}(\mathrm{Rot}_\theta) K(\phi,r)\rho_\mathrm{in}(\mathrm{Rot}_\theta^{-1})
\]






\paragraph{Calibration and Sharpness of Probabilistic Prediction.} 
It is desirable for a probabilistic prediction to be both \textit{calibrated} and \textit{sharp}. A model is \textit{calibrated} when the predicted probability correspond to the true probability of an event. In forecasting, calibration is often measured with \textit{coverage}, the probability of ground truth ${Y}$ falls into prediction region of confidence $\Hat{C}^{1-\alpha}$, $1-\alpha \in [0,1]$. 
\begin{align}
    P(Y \in \Hat{C}^{1-\alpha}) \geq 1- \alpha
    \label{eq:coverage}
\end{align}
Note that one can always obtain a calibrated prediction region by setting $\Hat{C}^{1-\alpha}$ to be the entire $Y$ space. We introduce the concept of  \textit{Sharpness}, which refers to the size of the prediction region $|\Hat{C}^{1-\alpha}|$. The balance between calibration and sharpness can be measured by a class of metrics called proper scoring rules \citep{properscoring}.
\vskip -0.1 in

\section{PECCO}

\begin{figure}[ht!]
  \centering
  \includegraphics[width=1\linewidth]{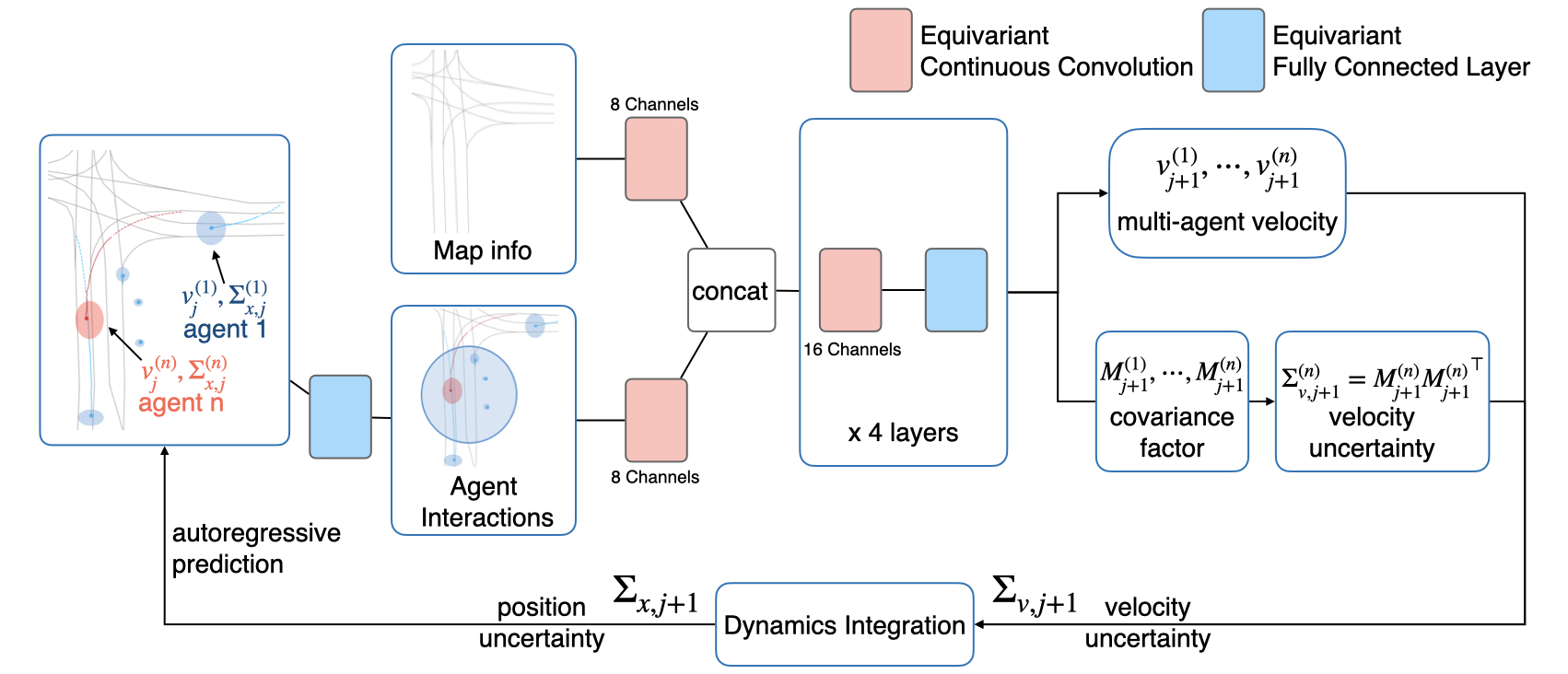}
  \caption{Overview of PECCO's model architecture. Agent trajectories consisting of velocities and position uncertainties are encoded along with map information by equivariant continuous convolution and fully connected layers. The model outputs $v_{j+1}$ and $\Sigma_{v,j+1}$ for all agents, which we use to calculate position uncertainty $\Sigma_{x,j+1}$ via dynamics. The model takes in the forecast and predicts autoregressively.} 
  \label{fig:model}
\end{figure}

\subsection{Problem Setup}
 Given  past trajectories of $n$ agents over $t$ time steps $\{x_{j}^{(1)}, x_{j}^{(2)}, \cdots, x_{j}^{(n)} \}_{j=1}^t $, where  ${x}^{(i)}_{j} \in \mathbb{R}^2$ , and the environmental context information $\mathbf{e}$ including marker positions of map lane boundaries, we model the probability distribution of agents' positions over  $k$ future time steps  as $p_\theta( {x}_{t+1:t+k} | {x}_{1:t}, \mathbf{e})$, with ${x}_j = ({x}^{(1)}_{j},\cdots, {x}^{(n)}_{j})$ being the positions of all agents at time step $j$. We introduce  PECCO, a deep learning model that leverages rotational equivariance to produce probabilistic forecasts.

The high-level architecture of our model is illustrated in Figure \ref{fig:model}. PECCO takes as input the positions of all agents ${x}_{1:t}$ in the past, and the covariance matrix $\Sigma_{x,j}$ at time $j$. It outputs the probability distribution of each agents' velocity as a 2-D Gaussian $\T{N}(\mu_{v,j+1}^{(i)}$, $\Sigma_{v,j+1}^{(i)})$ for the next time step. The velocity distribution is then integrated into a position distribution $\mathcal{N}$($\mu_{x, j+1}^{(i)}$, $\Sigma_{x, j+1}^{(i)}$) through dynamics integration. PECCO predicts the future $k$ timesteps autoregressively. The output distributions are guaranteed to be rotational equivariant by our model implementation.

\subsection{Probabilistic Symmetry through Equivariant Neural Networks} 

Rotational equivariance effectively reduces the dimension of the data space by placing different samples in the same equivalence class. This improves data coverage for better probabilistic modeling. 

Intuitively, real-world trajectory dynamics has intrinsic symmetry. That is, if past trajectories and environmental data such as the map is rotated, then the probability of a rotated future trajectory will be equally likely. We can model the probability density function $p_\theta$ as an invariant function of its inputs as in \autoref{eqn:invariant_prob}. Here, each past and future position $x^{(i)}_j \in \mathbb{R}^2$ is transformed according the standard representation $\rho_1$.  
\begin{equation} \label{eqn:invariant_prob}   
    p_\theta( {x}_{t+1:t+k} | {x}_{1:t}, \mathbf{e}) = p_\theta( g{x}_{t+1:t+k} | g{x}_{1:t}, g\mathbf{e}) \quad \forall g\in \SOtwo
\end{equation}


In order to implement \autoref{eqn:invariant_prob}, we 
assume future positions  follow a multivariate normal distribution ${x}^{(i)}_{j} \sim N(\mu^{(i)}_{x,j},\Sigma^{(i)}_{x,j})$. This is a common assumption in trajectory forecasting literature \citep{trajpredsurvey} and provides a convenient parametric form for optimization and reasoning. In the following expositions we omit the underscore $x$ for simplicity.

We aim to construct an equivariant neural network $f_\theta$ that outputs the parameters $\mu^{(i)}_{j}$ and $\Sigma^{(i)}_{j}$ autoregressively, taking as input probability distributions over the positions of all agents in the past $k$ time steps
\[
\mu_{j+1},\Sigma_{j+1} = f_\theta(\mu_{j-t+1:j}, \Sigma_{j-t+1:j}, \mathbf{e}). 
\]
where $\mu_{j} = (\mu_{j}^{(1)},\ldots,\mu_{j}^{(n)})$ and $\Sigma_j = (\Sigma_{j}^{(1)},\ldots,\Sigma_{j}^{(n)})$ and $\mathbf{e}$ denotes environmental information.  

In this case, the equivariance of $f_\theta$ leads to the desired invariance of $p_\theta$. This may be seen as a partial evaluation or currying of the conditional probability density function which has the effect of transforming invariance to equivariance. The following proposition relates equivariant networks with probabilistic symmetry. See Appendix \ref{app:equiv_dist} for a proof.
\begin{proposition}[One step equivariance implies $n$-step equivariance] \label{prop:equivtoinvar}
If the one-step probabilistic forecasting model $f_\theta$ is $G$-equivariant, then the probability distribution $p_\theta( x_{t+1:t+k} | x_{1:t}, \mathbf{e})$ is invariant as in \autoref{eqn:invariant_prob}.
\end{proposition}

In order to enforce $\SOtwo$-equivariance for $f_\theta$, the following proposition describes how the mean and covariance matrix for a 2-D  Gaussian  transforms under a rotation of $\mathbb{R}^2$. 

\begin{proposition}[SO(2) equivariance of multivariate Gaussian]  \label{prop:equivnormal}
Given multivariate normal distribution $\mathcal{N}(\mu,\Sigma)$ over $\mathbb{R}^2$ with probability density function $p_{\mu, \Sigma}$ and $g\in \SOtwo$, then $\mathcal{N}(g\mu, g\Sigma g^T)$ is also a multivariate normal distribution and $p_{g\mu, g\Sigma g^T} (v) = p_{\mu, \Sigma} (g^{-1}v)$ for all $v \in \mathbb{R}^2$.
\end{proposition}  

To ensure the covariance matrix of $f_\theta$ is positive-definite and symmetric, i.e. $\Sigma \in \mathrm{PosDefSym}_2(\mathbb{R})$, we make use of the following fact: 

\begin{proposition}[Equivariant maps constructing positive-definite symmetric matrices] \label{prop:pds}
The map
\begin{align}
    \varphi \colon \mathrm{GL}_2(\mathbb{R}) \to \mathrm{PosDefSym}_2(\mathbb{R}) , \quad 
    M \mapsto M M^T \label{eq:phi}
\end{align}
is surjective and equivariant. That is, for $g \in \SOtwo$ we have
\begin{align}
    \varphi(gM) &= g \varphi(M) g^{T} \label{eq:act_on_gl2}. 
\end{align}
Moreover, $\varphi$ admits a one-sided inverse which is also equivariant,
\begin{align}
        \psi \colon \mathrm{PosDefSym}_2(\mathbb{R}) \to \mathrm{GL}_2(\mathbb{R}), \quad    
        \Sigma \mapsto Q \Lambda^{\frac{1}{2}}
\end{align}
where $Q\Lambda Q^T$ is the eigendecomposition of $\Sigma$ and $Q$ is orthogonal. Together, $\varphi (\psi(\Sigma)) = \Sigma$.
\end{proposition}

As a consequence of Proposition \ref{prop:equivnormal} and Proposition \ref{prop:pds}, we can ensure that the predicted distribution transforms correctly under equivariance by (1) predicting an intermediate matrix $M$ for covariance, and (2) constraining $\tilde{f}_\theta$ to be equivariant with respect to the action in \autoref{eq:act_on_gl2}.
\begin{equation} \label{eq:ftheta}
\begin{aligned}    
    \mu_{j+1},M_{j+1} = \tilde{f}_\theta(\mu_{j-k:j}, M_{j-k:j}, \mathbf{e}), \;
    \Sigma_{j+1} = M_{j+1}M_{j+1}^T
\end{aligned}
\end{equation}
In this case, $\SOtwo$ acts by transforming the the columns of $M$ independently as vectors in $\mathbb{R}^2$.  Thus the data $(\mu_j^{(i)},\Sigma_j^{(i)})$ for each agent and time step is comprised of 3 copies of the standard representation $\rho_1$ as defined in Section \ref{subsec:symmetry}.  Given this $\SOtwo$ action, we can enforce $\SOtwo$-equivariance in the neural network $\tilde{f}_\theta$. Implementation details of  of the equivariant layers are provided in Appendix \ref{app:matrix}.




\subsection{Dynamics Integration (dyna)} 
\label{sec:dyn_inte}
Instead of predicting the position directly, PECCO outputs a Gaussian distribution over \emph{velocity} as $\mathcal{N}(\mu_{v,j}, \Sigma_{v,j})$. More specifically, it predicts $(\mu_{v,j}, M_{v,j})$ at each time step and the covariance matrix $\Sigma_{v,j}$ is calculated as in Proposition \ref{prop:pds}. However, we want to obtain the  uncertainty over \textit{position} as $\Sigma_{x,j}$ and perform autoregressive forecasting. We leverage dynamics integration to propagate the uncertainty from velocity to position.

Assuming that all agents in the system can be approximated as linear discrete time dynamics ${x}_{j+1} = {x}_j + \Delta t \cdot  {v}_j $, we can obtain the uncertainty of predicted position $\Sigma_{x, j+1}$ by 
\[ 
\Sigma_{x, j+1}
            = \Sigma_{x, j} + (\Delta t)^2 \, \Sigma_{v, j} + 2 \Delta t \cdot \text{cov}(x_j, v_j).
\]
We assume that the cross covariance matrix $\mathrm{cov}(x_j, v_j) $ is zero for simplicity 
following previous works \citep{trajectron++}. During training, gradients are calculated after entire trajectory is predicted autoregressively. A consequence of the additive uncertainty setup is that it enforces the uncertainty 
to grow monotonically over time, creating a ``cone of uncertainty".

\section{Experiments}

We show that our model produces accurate and more calibrated probabilistic forecast compared to baseline models on one synthetic and two real world trajectory prediction datasets: interacting particles, autonomous vehicle, and pedestrian movement. 
\vspace{-0.5em}
\subsection{Baselines}
\begin{itemize}
\itemsep0em 
   \item \texttt{LSTM-NLL} (variation of \cite{sociallstm}): An encoder-decoder LSTM model that predicts the mean and variance of a Gaussian distribution, optimizing likelihood of data. We also train a version with random rotation data augmentation \texttt{LSTM-aug}.
    \item \texttt{CtsConv} \citep{ummenhofer2019lagrangian}: Continuous convolution over point cloud data for trajectory prediction, a non-equivariant counterpart of PECCO. \texttt{CtsConv-aug} is trained with a data augmentation step where we randomly rotate the scenes. 
     \item \texttt{Multiple Futures Prediction (MFP)} \citep{mfp}: A encoder-decoder model for multimodal probabilistic trajectory forecasts.
     \item \texttt{Trajectron++} \citep{trajectron++}: State-of-the-art probabilistic trajectory prediction model with graph representation of agent interactions and conditional VAE architecture.
\end{itemize}
\vspace{-1em}
\subsection{Evaluation Metrics.}
\begin{itemize}
\itemsep0em 
\item \textit{Minimum Average/Final Displacement Error} (minADE$_6$, minFDE$_6$):  average $l_2$ displacement error over $k$ steps, or average displacement of the final step, between predicted and ground truth trajectories. We report the minimum over 6 samples for probabilistic models. 
\item \textit{Negative Log Likelihood} (NLL): NLL of ground truth trajectories under predicted distributions.
\item \textit{Energy Score} (ES) \citep{properscoring}: a proper scoring rule to measure calibration and sharpness of the predicted distribution $P$. The energy score for a distribution $P$ and the ground truth data $x$ is defined as: $\text{ES}(P, x) = E_{X \sim P} \| X - x \| - \frac{1}{2} E_{ X, X' \sim P} \| X - X' \| $. Here $X$ and $X'$ are independent samples from $P$. 
\item \textit{Coverage}: The empirical estimate of probability of the true value lying in the predicted interval, defined in  \eqref{eq:coverage}. We report the coverage of 90\% quantile of the predicted Gaussian. The  prediction is more \textit{calibrated} if the coverage is closer to 90\%.
\end{itemize}

\vspace{-1em}
\subsection{Datasets.}
The particle dataset is a synthetic dataset of 5 particles interacting in spring dynamics \citep{nri} with dynamics noise. The models predict 20 time steps given 30 steps as input.
The Argoverse autonomous vehicle motion forecasting \citep{argoverse} is a widely used vehicle trajectory prediction benchmark. The task is to predict 3 second trajectories based on all vehicle history in the past 2 seconds recorded at 10Hz. 
TrajNet++ \citep{pedestriandata} is a popular pedestrian trajectory benchmark with a focus on agent-agent interaction scenarios. The task is to predict 12 time steps for agents given 9. We refer the reader to Appendix \ref{app:exp} for data and training details. 

\vspace{-1em}
\subsection{Experimental Results}

\begin{wraptable}{r}{5.5cm}
\vspace{-1.0em}
\resizebox{\textwidth}{!}{
    \begin{tabular}{c|c c c}
     \toprule
  \textbf{Model} & MSE $\downarrow$ & NLL$\downarrow$  & ES  $\downarrow$\\ 
  \midrule
  LSTM & .016 &  -1.61 & 1.041 \\
  CtsConv & .010 &  -0.81  & 0.667 \\
  PECCO & \textbf{.004} & \textbf{-0.83} & \textbf{0.467} \\
    \bottomrule
    \end{tabular}
\caption{PECCO outperforms baseline models in all metrics on the synthetic particles dataset. }}
\label{tab:particles}
\vspace{-0.5em}
\end{wraptable}

\begin{table}[t!]
\centering   \resizebox{\textwidth}{!}{
\begin{tabular}{ c|c c c c | ccc } 
 \toprule
  \textbf{Model} & minADE${}_6 \downarrow$ & minFDE${}_6 \downarrow$ & NLL$\downarrow$ & ES $\downarrow$& Cov@1s(\%) & Cov@2s & Cov@3s\\ 
\midrule
 \multicolumn{8}{c}{Argoverse}\\
 \midrule
LSTM-NLL   & 1.64 \plusminus{.05} & 4.17 \plusminus{.10} & 3.07 \plusminus{.08} & 2.31 \plusminus{.54} & 8.8 \plusminus{0.7} & 8.5 \plusminus{0.7} & 7.0 \plusminus{0.8} \\
LSTM-NLL-aug    & 1.61 \plusminus{.02} & 4.15 \plusminus{.08} & 2.78 \plusminus{.03} & 1.99 \plusminus{.46} & 10.1 \plusminus{1.5} & 10.5 \plusminus{1.0} & 9.8 \plusminus{1.9} \\
CtsConv-NLL  & 1.74 \plusminus{.03} & 4.43 \plusminus{.06} & 29.1 \plusminus{2.2} & 6.71 \plusminus{.70} & 6.3 \plusminus{2.2} & 0.02 \plusminus{.01} & 0.01 \plusminus{.01}\\
CtsConv-NLL-aug  & 1.66 \plusminus{.02} & 4.23 \plusminus{.06} & 11.81 \plusminus{.01} & 5.10 \plusminus{.35} & 11.9 \plusminus{2.1} & 1.7 \plusminus{0.5} & 0.02 \plusminus{.01} \\
Trajectron++   & 1.83 \plusminus{.02} & 3.85 \plusminus{.07} & \textbf{2.48} \plusminus{.27} & 3.92  \plusminus{.61} & 45.5 \plusminus{5.3} & 37.6 \plusminus{3.2}  & 34.9 \plusminus{2.5} \\
MFP  & 1.53 \plusminus{.04} & 3.77 \plusminus{.06} & 3.56 \plusminus{.02} & 2.33 \plusminus{.21} & 51.3 \plusminus{5.1} & 33.0 \plusminus{4.9} & 8.3 \plusminus{4.8} \\
\midrule
PECCO   & \textbf{1.39} \plusminus{.02}  & \textbf{3.41} \plusminus{.03} & 4.26 \plusminus{0.1} & \textbf{1.54} \plusminus{.16} & \textbf{74.9} \plusminus{0.6} & \textbf{78.6} \plusminus{2.8} & \textbf{84.5} \plusminus{2.9}\\
\toprule
 \multicolumn{8}{c}{TrajNet++}\\
  \midrule
LSTM-NLL-aug & 0.85 \plusminus{.02} & 1.64 \plusminus{.03} & 2.78 \plusminus{.02} & -0.28 \plusminus{.09}  & 29.0 \plusminus{4.3} & 23.2 \plusminus{4.2} & 23.7\plusminus{3.9} \\
CtsCov-NLL & 1.08 \plusminus{.02} & 2.36 \plusminus{.09} & 5.33 \plusminus{.08} & 1.67 \plusminus{.13}  & 43.8 \plusminus{10.6}  & 20.7 \plusminus{5.2} & 12.2 \plusminus{6.7} \\
CtsCov-NLL-aug &  0.92 \plusminus{.01 } & 1.76 \plusminus{.03} & 6.74 \plusminus{.21} & 1.42 \plusminus{.11}  & 62.1 \plusminus{3.3}  & 36.3 \plusminus{4.9} & 34.1 \plusminus{5.8} \\
Trajectron++  & 1.14 \plusminus{.03} & 2.31 \plusminus{.05} & 2.83 \plusminus{.12} & 0.98 \plusminus{.17} & 50.2 \plusminus{2.2} & 45.8 \plusminus{3.5} & 32.9 \plusminus{3.5}\\
MFP & 0.85 \plusminus{.02}  & 1.70 \plusminus{.04}  & \textbf{2.20} \plusminus{.04}  & 0.67 \plusminus{.08} & 79.1 \plusminus{4.3}  & 32.5 \plusminus{3.1}  & 22.8 \plusminus{3.2} \\
\midrule
PECCO & \textbf{0.59} \plusminus{.12}  & \textbf{1.06} \plusminus{.17} & 2.37 \plusminus{.04} & \textbf{-0.73} \plusminus{.10} &  \textbf{80.8} \plusminus{4.5} &  \textbf{85.9} \plusminus{2.3} &  \textbf{94.5} \plusminus{3.0} \\
\bottomrule
\end{tabular}
}
\caption{Performance comparison on benchmark datasets Argoverse and TrajNet++. Cov@$k$s(\%) refers to coverage at the $k$ second mark; prediction is more \textit{calibrated} if closer to 90\%. PECCO is more accurate and calibrated compared to non-equivariant baseline models.} 
\label{tab:big-table}
\end{table}
\begin{figure}
    \centering
    \subfigure[LSTM]{\label{fig:noneq0}
      \includegraphics[width=.32\textwidth]{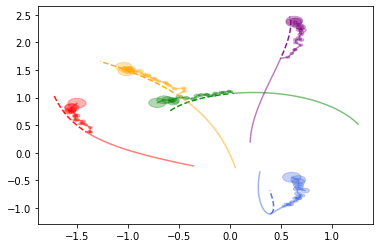}}
    \subfigure[CstConv]{\label{fig:eq0}
      \includegraphics[width=.32\textwidth]{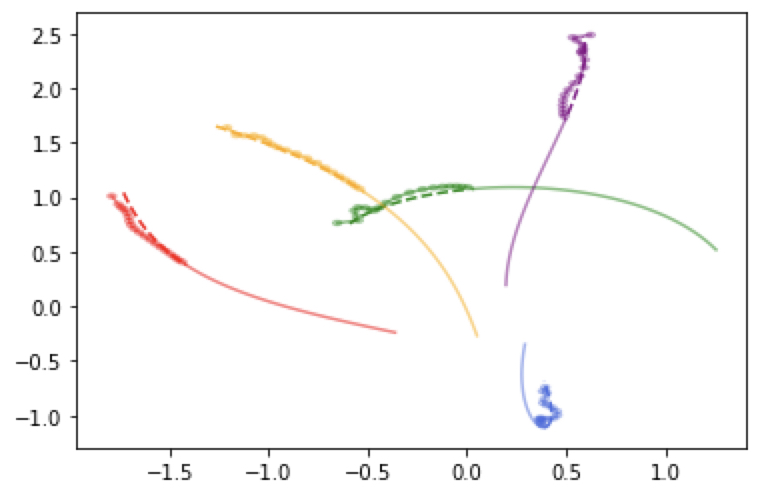}}
          \subfigure[PECCO]{\label{fig:eq0}
      \includegraphics[width=.32\textwidth]{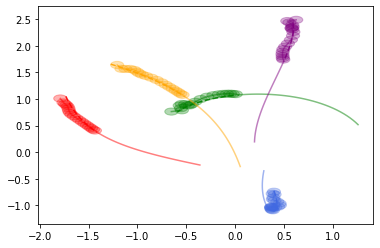}}
    \caption{Comparison of prediction results between baselines. The solid lines are input timesteps to the models, the dotted lines ground truth, and the circles the 90\% confidence regions. We can see that PECCO achieves accurate results while maintaining good coverage.}
    \label{fig:particle}
\end{figure}

\textbf{Particles dataset.} We present experimental result for the synthetic particle dataset in Table 1. We visualize a test sample in \autoref{fig:particle} to qualitatively illustrate PECCO's improved accuracy and calibration. 

\noindent \textbf{Argoverse and TrajNet++.}
\autoref{tab:big-table} presents experimental results on two benchmark datasets. PECCO achieves better regression accuracy compared to non-equivariant baseline in terms of minADE and minFDE, with a notable 9\% improvement in minADE over the the best performing baseline, MFP. PECCO's improved probability coverage allows for more diverse sampling and hence can produce trajectories closer to ground truth. 

PECCO's probabilistic predictions are able to achieve consistently better coverage compared to other methods whose coverage deteriorates over time. Comparing LSTM-NLL and CtsConv-NLL with their augmented counterparts, LSTM-NLL-aug and CtsConv-NLL-aug, we can see that data augmentation through rotation improves both accuracy and calibration. PECCO leverages this symmetry to improve accuracy, maintain good calibration, and converge faster (\autoref{fig:dataeff}).

Figure \ref{fig:arg-experiment} visualizes a typical situation to illustrate this difference. We plot the predicted distribution at 10, 20, and 30 time steps of prediction (1 time step is 0.1 seconds). We can see the probable region predicted by LSTM-NLL explodes at timesteps 20 and 30, whereas CtsConv-NLL and MFP tend to be overconfident in their predictions. PECCO is able to predict a Gaussian that covers both cases of staying in the lane and changing to the right lane.

\begin{figure}
    \subfigure[LSTM-NLL]{
  \includegraphics[width=.23\textwidth]{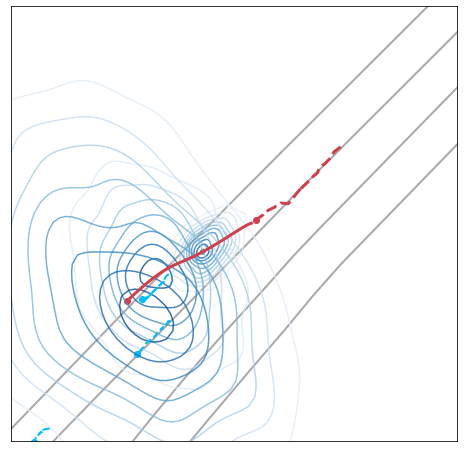}
  \label{fig:sub1}}
\subfigure[CtsConv-NLL]{
  \includegraphics[width=.23\textwidth]{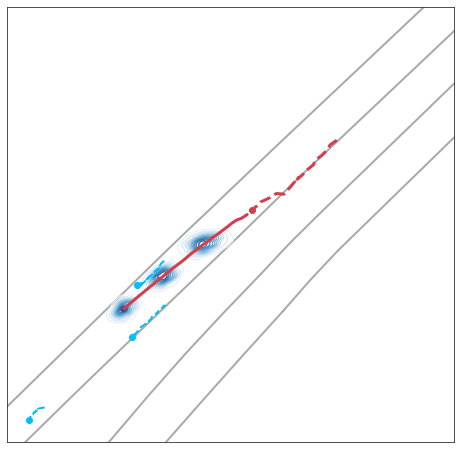}
  \label{fig:sub3}}
  \subfigure[MFP]{
  \includegraphics[width=.23\textwidth]{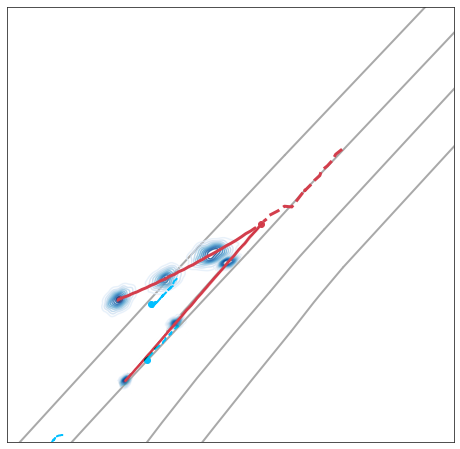}
  \label{fig:sub3}}
\subfigure[PECCO]{
  \includegraphics[width=.23\textwidth]{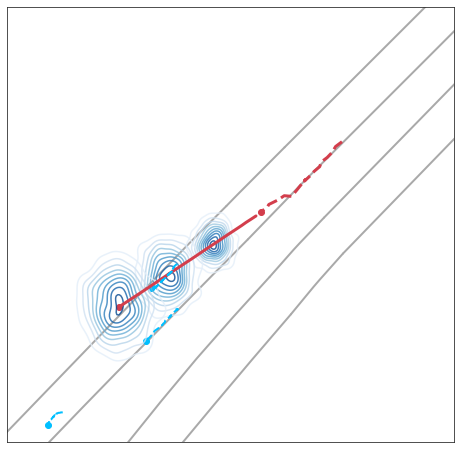}
  \label{fig:sub4}}
    \caption{Comparison of uncertainty predicted at a lane change. Red trajectories are the agent of interest in the same scenario. Note how the LSTM predicted uncertainty explodes after a few time steps, while CtsConv and MFP has overconfident distributions. PECCO is able to model possibility of both staying and lane change.}
    \label{fig:arg-experiment}
\end{figure}

\begin{figure}
    \centering
      \includegraphics[width=.9\textwidth]{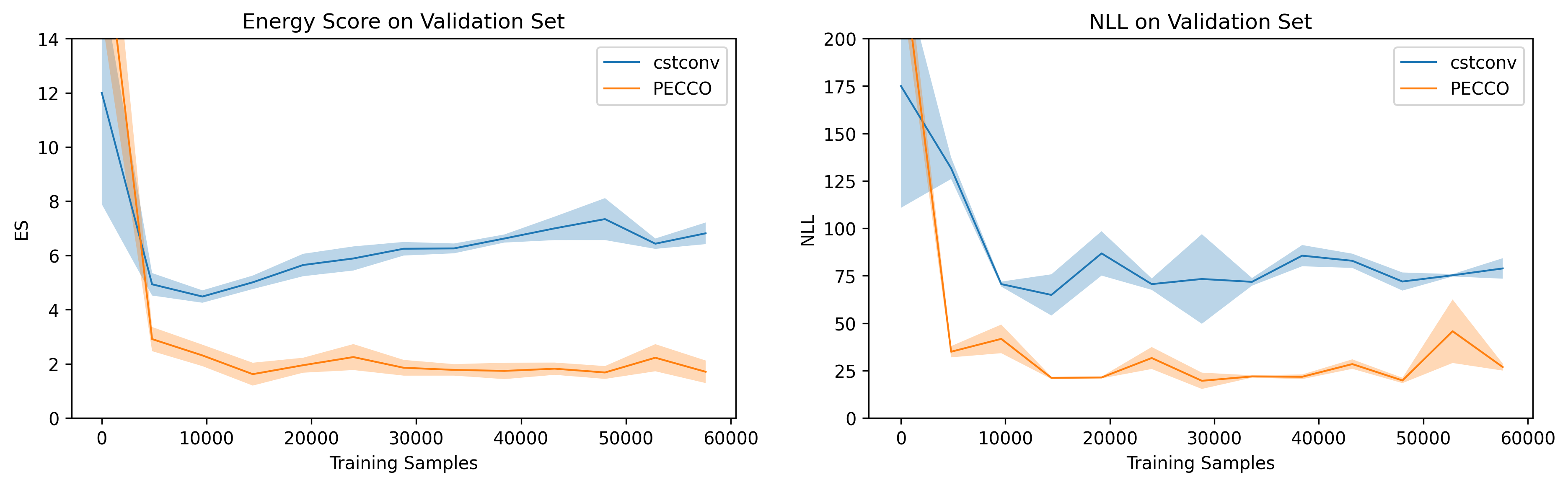}
    \caption{With equivariance, PECCO is able to achieve better energy score and NLL with fewer data samples, compared to its non-equivariant counterpart.}
    \label{fig:dataeff}
\end{figure}

\begin{table*}[h]
\centering\resizebox{\textwidth}{!}{
\begin{tabular}{ c|c c c| ccc } 
 \toprule
  \textbf{Model} & minADE${}_6 \downarrow$ & minFDE${}_6 \downarrow$ & MRS $\downarrow$ & Cov@1s(\%) & Cov@2s & Cov@3s\\ 
\midrule
Conformal LSTM & 2.45 \plusminus{0.09}  & 4.68 \plusminus{0.15} & 198.1 \plusminus{12.0} & 90.1 \plusminus{0.1} & 92.7 \plusminus{0.1} & 92.8 \plusminus{0.1}\\
Conformal ECCO & 1.96 \plusminus{0.06} & 4.32 \plusminus{0.10} & 220.9 \plusminus{8.1} & \textbf{90.0} \plusminus{0.1} & \textbf{90.1} \plusminus{0.1} & \textbf{90.0} \plusminus{0.1} \\
\midrule
PECCO   & \textbf{1.39} \plusminus{.02}  & \textbf{3.41} \plusminus{0.03} & \textbf{8.52} \plusminus{0.16} & 74.9 \plusminus{0.2} & 93.6 \plusminus{0.8} & 92.5 \plusminus{0.9}\\
\bottomrule
\end{tabular}}
\caption{Comparison with conformal prediction methods on Argoverse dataset. PECCO produces a parametric distribution with a tighter confidence region (small MRS), whereby achieving better regression accuracy while maintaining competitive coverage.}
\label{tab:conformal}
\end{table*}

\subsection{Model Analysis}
%
\paragraph{Comparison with data augmentation and cannonicalization.}
Data augmentation and cannonicalization are popular methods to implicitly exploit  symmetry in trajectory data. In Table 4 we compare PECCO to its nonequivariant counterparts with the addition of data augmentation and cannonicalization. Equivariance dramatically improves prediction performance in all aspects, especially calibration.

\begin{wraptable}{r}{7.5cm}
\vspace{-1.0em}
\centering \resizebox{\textwidth}{!}{
\begin{tabular}{ c | c c c c } 
 \toprule
  & minA/FDE$_6$ & NLL & Cov(\%) \\ 
 \midrule
  CtsConv & 1.74 / 4.43  & 29.13 & 2.2 \\
+cannon & 1.66 / 4.28 & 17.46  &  4.5  \\
+aug & 1.67 / 4.23  & 11.81  &  17.1 \\
equivariant & \textbf{1.39} / \textbf{3.41} & \textbf{4.26}  & \textbf{87.5} \\
\bottomrule
 \end{tabular}
\caption{Data Augmentation Comparison on the Argoverse Dataset.}}
\label{tab:aug-ablation}
\end{wraptable}

\vspace{-0.5em}
\paragraph{Dynamics Integration Ablation.}
Dynamic integration (dyna) introduced in Section \ref{sec:dyn_inte} enforces the uncertainty to grow monotonically over time. As an ablative study, Table 5 show that PECCO with  dynamic integration has much better calibration compared to if without; for the Argoverse dataset, the improved calibration also informs better performance on minADE/FDE.

\begin{wraptable}{r}{7.5cm}
\vspace{-1.0em}
\centering \resizebox{\textwidth}{!}{
\begin{tabular}{ c|c c c  } 
\toprule
\textbf{Argoverse} & minA/FDE$_6$ & NLL & Cov(\%) \\ 
\midrule
no-dyna & 1.52 / 3.76 & 9.72  & 38.6 \\
dyna & \textbf{1.39} / \textbf{3.41} & \textbf{4.26}  &  \textbf{87.5} \\
\toprule
\textbf{Pedestrian} & minA/FDE$_6$ & NLL & Cov(\%) \\ 
\midrule
no-dyna & \textbf{0.72} / 2.12 & 4.71 & 39.6  \\
dyna & 0.73 / \textbf{1.98} & \textbf{2.37}  & \textbf{83.7} \\
\bottomrule
\end{tabular}
\caption{Dynamics Integration (dyna) Ablation: using dyna encourages the uncertainty to grow over time and improves coverage.}}
\label{tab:dynamics-ablation}
\end{wraptable}

\vspace{-0.5em}

\paragraph{Comparison with Conformal Prediction.}
We compare to two conformal prediction baselines \citep{2021conformal} (Appendix \ref{app:conformal} for details) in Table 3. Conformal methods achieve guaranteed $\geq 90\%$ coverage, but suffer in prediction accuracy due to having to split training data for calibration. Since conformal prediction does not output a distribution, we use another proper scoring rule Mean Regional Score (2d extension of mean interval score in \cite{properscoring}, see Appendix B) as metric. The conformal regions are much larger (higher MRS), which is less desirable for downstream decision making tasks.
\section{Conclusion}
In this work we propose Probabilistic Equivariant Continuous Convolution (PECCO), a novel  multi-agent probabilistic prediction method for improving uncertainty quantification. We design an equivariant neural network under which the predicted distributions transform correspondingly as inputs are transformed. We introduce the Energy Score metric to bring attention to the calibration of multivariate probabilistic forecasts. By leveraging equivariance, PECCO produces more accurate and calibrated probabilistic forecasts compared to existing methods on both synthetic and real-world datasets.

\acks{This work was supported in part by U.S. Department Of Energy, Office of Science, Facebook Data Science Research Awards, U. S. Army Research Office under Grant W911NF-20-1-0334, and NSF Grants \#2134274 and \#2146343. }

\bibliography{ref.bib}
\pagebreak
\appendix

\section{Proofs}
\label{app:equiv_dist}

\subsection{Proofs for Equivariance of Gaussians [Proposition \ref{prop:equivnormal}]}

In this section, we prove that the Gaussian probability density function is $\SOtwo$-equivariant.

\begin{prop*}
 Given multivariate normal distribution $\mathcal{N}(\mu,\Sigma)$ over $\mathbb{R}^2$ with probability density function $p_{\mu, \Sigma}$ and $g\in \SOtwo$, then $p_{g\mu, g\Sigma g^T} (v) = p_{\mu, \Sigma} (g^{-1}v)$ for all $v \in \mathbb{R}^2$.
\end{prop*}

\begin{proof}
\begin{align*}
P_{g\mu, g\Sigma g^T}(v) &= \frac{1}{2 \pi \det{(g\Sigma g^T)}} \exp{(-\frac{1}{2}(v-g\mu)^T (g \Sigma g^T)^{-1} (v - g\mu))} \\
 &= \frac{1}{2\pi \det{(\Sigma})} \exp{(-\frac{1}{2}(g^{-1}v-\mu)^T g^T g^{-T} \Sigma^{-1} g^{-1} g (g^{-1}v -\mu))} \\
&= \frac{1}{2\pi \det{(\Sigma})} \exp{(-\frac{1}{2}(g^{-1}v-\mu)^T \Sigma^{-1} (g^{-1}v -\mu))}\\ 
&= P_{\mu, \Sigma}(g^{-1}v) 
\end{align*}
\end{proof}

\begin{prop*} 
Given a normal distribution $\mathcal{N}(\mu,\Sigma)$, its group-actioned $\mathcal{N}(g\mu, g\Sigma g^T)$ is also a valid normal distribution. 
\end{prop*}

\begin{proof}
We prove the proposition by proving that $g\Sigma g^T$ is a symmetric positive definite matrix.

1) \textit{(Symmetry)} $(g\Sigma g^T)^T = g^{TT} \Sigma^T g^T = g\Sigma g^T$.

2) \textit{(Positive Definite)}  Let $v\in \mathbb{R}^2 $ and $v\neq 0$. Let $w = g^T v$, note that $w\neq 0$. Then $v^T g\Sigma g^T v =  w^T \Sigma w >0$, which implies $\Sigma$ is a positive definite matrix.  
\end{proof}

\subsection{Equivariance of construct for positive definite and symmetric $\Sigma$ [Proposition \ref{prop:pds}]}


Note that every positive definite symmetric matrix can be written as $\Sigma = MM^T$ for some $M\in GL_2(\mathbb{R})$, so we can represent $\Sigma$ by a map $\varphi$ where
\begin{align*}
       \varphi \colon \mathrm{GL}_2(\mathbb{R}) &\to \mathrm{PosDefSym}_2(\mathbb{R})  \\ 
    M &\mapsto M M^T
\end{align*}

\begin{prop*}
$\varphi$ is $\SOtwo$-equivariant.
\end{prop*}
\begin{proof}
$\varphi (gMg^T) = gMg^T (gMg^T)^T = gMg^T g^{TT} M^T g^T = gMM^Tg^t = g\varphi(M)g^T $
\end{proof}

We defined the one-sided inverse of $\varphi$ as $\psi$:
\begin{align*}
        \psi \colon \mathrm{PosDefSym}_2(\mathbb{R}) &\to \mathrm{GL}_2(\mathbb{R})    \\ 
        \Sigma &\mapsto Q \lambda^{\frac{1}{2}}
\end{align*}
where $Q\lambda Q^T$ is the eigendecomposition of $\Sigma$ and $Q$ is orthogonal.   That is, $\varphi (\psi(\Sigma)) = \Sigma$.
\begin{prop*}
$\psi$ is $\SOtwo$-equivariant.
\end{prop*}
\begin{proof}
For symmetric positive definite matrix $\Sigma$, its eigenvector $v_i$ and eigenvalue $\lambda_i$ follows the eigenequation
\[
    \Sigma v_i = \lambda_i v_i .
\]
For $\forall g \in \SOtwo$, 
\[
  g\Sigma v_i = \lambda_i g v_i \implies g\Sigma g^{-1} g v_i = \lambda_i g v_i \implies (g\Sigma g^{-1}) (g v_i) = \lambda_i (g v_i) ,
\]
so $gv_i$ is an eigenvector to $g\Sigma g^{-1}$, with corresponding eigenvalue $\lambda_i$.

Therefore, 
\[
\psi (g\Sigma) = gQ\Lambda^{\frac{1}{2}} = g\psi (\Sigma).
\]
\end{proof}

\subsection{Proof of Proposition \ref{prop:equivtoinvar}: The Equivariant Model Gives an Invariant Conditional Distribution. }

We now prove Proposition \ref{prop:equivtoinvar}.

\begin{prop*}[Proposition \ref{prop:equivtoinvar}]
If the one-step probabilistic forecasting model $f_\theta$ is $G$-equivariant, then the conditional probability distribution $p_\theta( \mathbf{x}_{t+1:t+k} | \mathbf{x}_{1:t}, \mathbf{e})$ is invariant as in \autoref{eqn:invariant_prob}.
\end{prop*}

\begin{proof}
Since $\mathbf{x}_{1:t}$ are known, we define $\Sigma_{1:t} = \epsilon \mathrm{Id}$ for $\epsilon$ small approximating a delta distribution at $\mu_{1:t} = \mathbf{x}_{1:t}$.
We then evaluate $\mu_{j+1},\Sigma_{j+1} = f_\theta(\mu_{j-t+1:j}, \Sigma_{j-t+1:j}, \mathbf{e})$ autoregressively for $j=1,\ldots,k$.  Finally define 
\begin{equation}
\begin{aligned}
    p_\theta( \mathbf{x}_{t+1:t+k} | \mathbf{x}_{1:t}, \mathbf{e}) &=
    \prod_{j=1}^k  p_\theta( \mathbf{x}_{t+j} | \mathbf{x}_{j:t+j-1}, \mathbf{e}) \\
    &= \prod_{j=1}^k p_{N(\mu_{t+j},\Sigma_{t+j})}(\mathbf{x}_{t+j}). 
\end{aligned}
\label{eq:timeseriespdf}
\end{equation}

For $g \in \SOtwo$, we want to show 
\[
    p_\theta( g\mathbf{x}_{t+1:t+k} | g\mathbf{x}_{1:t}, g\mathbf{e}) = p_\theta( \mathbf{x}_{t+1:t+k} | \mathbf{x}_{1:t}, \mathbf{e}).
\]
  Thus for $g\mathbf{x}_{1:t}, g\mathbf{e}$ we initialize $g \mu_{1:t},g\Sigma_{1:t}g^{T}$ since $g\Sigma_{1:t}g^T = gg^T\Sigma_{1:t}= \Sigma_{1:t}$. Applying $f_\theta$ repeatedly and invoking equivariance, we obtain 
\begin{align}
    g\mu_{t+j},g\Sigma_{t+j}g^{T} = f_\theta(g\mu_{j:t+j-1}, g\Sigma_{j:t+j-1}g^{T}, g\mathbf{e}). 
\end{align}
for $j =1,\ldots,k$.  Then by \autoref{eq:timeseriespdf}, 
\[
    p_\theta( g\mathbf{x}_{t+1:t+k} | g\mathbf{x}_{1:t}, g\mathbf{e}) =
    \prod_{j=1}^k p_{N(g\mu_{t+j},g\Sigma_{t+j}g^{T})}(g\mathbf{x}_{t+j}).
\]
Then by Proposition \ref{prop:equivnormal} this is 
\[
  \prod_{j=1}^k p_{N(\mu_{t+j},\Sigma_{t+j})}(g^{-1}g\mathbf{x}_{t+j}),
\]
which cancels $g^{-1}g = 1$, and applying \autoref{eq:timeseriespdf} again gives  $p_\theta( \mathbf{x}_{t+1:t+k} | \mathbf{x}_{1:t} \mathbf{e})$.  
\end{proof}

\subsection{Proof that MRS is a strictly proper scoring rule}
\label{app:properscoring}

Scoring rules are summary measures to evaluate probabilistic forecasts; they take in the forecasted distribution and the event or value that materializes and assign a numerical score. As defined in \cite{properscoring},  a scoring rule is proper if it is maximized when the forecaster recovers the ground truth distribution. It is strictly proper if the maximum is unique.

\begin{prop*} 
MRS is a strictly proper score.
\end{prop*}
\begin{proof}
For arbitrary function $h$, we first define the scoring rule
\[
 S(R;z) = -|R| + \frac{1}{\alpha}(|R| - |R(z)|)\mathbb{1}\{z\in R^C\} + h(z),
\]
which means if the forecaster quotes region $R$ at the level $\alpha\in (0,1)$ and $z$ materializes, then the score $S(R;z)$ will be rewarded, where $R$ and $R(z)$ are defined as ellipsoid parametric by $\mu$ and $\Sigma$. Then the expected score under the probability measure $P\in \mathcal{P}$ which $z$ follows is defined as
\[
 S(R;P) = \int S(R;z) \, dP(z). 
\]
For $P\in \mathcal{P}$, let $R_*$ denote the unique true $P$-region at level $\alpha$. We say that scoring rule $S$ is strictly proper if 
\[
 S(R_*;P) \geq S(R;P),
\]
for all region $R$ and for all probability measures $P\in \mathcal{P}$, where equality holds if and only if $R = R_*$. 

We identify $P$ with the associated distribution function so that $P(R_*^C)=\alpha$. If $R_* \subset R$, then
\begin{align*}
    S(R_*;P) - S(R;P) & = \int -|R_*| \, dP(z) + \frac{1}{\alpha} \int_{R_*^C} (|R_*| - |R(z)|) \, dP(z) \\
    & + \int |R| \, dP(z) - \frac{1}{\alpha} \int_{R^C} (|R| - |R(z)|) \, dP(z) \\
    & = -|R_*| + |R| + \frac{1}{\alpha}|R_*|P(R_*^C) - \frac{1}{\alpha}|R|P(R^C) - \frac{1}{\alpha}\int _{R_*^C \setminus R^C} |R(z)| \, dP(z) \\
    & \geq |R| - \frac{1}{\alpha}|R|P(R^C) - \frac{1}{\alpha}|R|(P(R_*^C)-P(R^C)) \\
    & = 0, 
\end{align*}
as it is supposed to be. If $R \subset R_*$, then an analogous argument applies. 

Putting $h(z) = 0$ and reversing the sign of the scoring rule, yields the negatively oriented regional score. Finally, the mean score of this oriented regional score is our MRS, which is then proved to be proper.
\end{proof}

\section{Equivariant Matrix Output Layer Construction}
\label{app:matrix}

Let 
\[
    \rho_1(\theta) = \begin{pmatrix}
    \cos{\theta} & -\sin{\theta} \\
    \sin{\theta} & \cos{\theta}
    \end{pmatrix} 
    \qquad\qquad
        M = \begin{pmatrix}
    a & b \\
    c & d
    \end{pmatrix}
\]
where $a,b,c,d$ are trainable parameters.
Consider input $f \colon S^1 \to \mathbb{R}$, then
\[
F(f) = \int_{\theta=0}^{2\pi} f(\theta) \rho_1(\theta) M \rho_1(-\theta) d\theta.
\]
Define $\rho_{reg}(\phi)(f)(\theta) = f(\theta-\phi)$.  
\vspace{5px}
\begin{prop*}
The convolution $F$ is equivariant, i.e.,
\[
F(\rho_{reg}(\phi) f) = \rho_1(\phi) F(f) \rho_1(-\phi).
\]
\end{prop*}
\begin{proof}
We compute
\begin{align*}
    F(\rho_{reg} (\phi)f) 
    &= \int_{\theta=0}^{2\pi}  \rho_{reg}(\phi)(f)(\theta) \rho_1(\theta) M\rho_1(-\theta) d \theta \\
    &= \int_{\theta=0}^{2\pi}  f(\theta-\phi) \rho_1(\theta) M\rho_1(-\theta) d \theta \\
    &= \int_{u=-\phi}^{2\pi-\phi} f(u) \rho_1(u+\phi)M\rho_1(-u-\phi)du \quad (\text{substituting } u = \theta - \phi) \\ 
    &= \rho_1(\phi) \cdot \left( \int_{u=-\phi}^{2\pi-\phi} f(u) \rho_1(u)M\rho_1(-u)du \right) \cdot \rho_1(-\phi) \\
    &=  \rho_1(\phi) F(f) \rho_1(-\phi). 
\end{align*} 
The last substitution follows from the fact the integrand is periodic of period $2\pi$. Since the integral is over the whole circle, this equal to taking the limits of integration to be $0$ and $2\pi$.
\end{proof}

\vspace{3px}

\paragraph{Discritization.}
If the function $f$ is discretized, such that $f_i = f(i2\pi/n)$.  Then the mapping $F$ is given by 
\[
F(f) = \sum_{i=0}^{n-1} f_i \rho_1(i2\pi/n) M \rho_1(-i2\pi/n).
\]

\paragraph{Compatibilitiy with matrix exponential}

For a $n \times n$ matrix $M$ define 
\[
\exp(M) = \sum_{k=0}^\infty M^k/k!.
\]
Matrix exponentiation is equivariant with respect to conjugation.  That is for $g \in \mathrm{GL}_n(\mathbb{R})$, 
\[
\exp(gMg^{-1}) = \sum_{k=0}^\infty (gMg^{-1})^k/k!. 
\]
Now note that
\[
(gMg^{-1})^k = (gMg^{-1}) (gMg^{-1}) \ldots (gMg^{-1}) = gMM \ldots Mg^{-1} = g M^k g^{-1}.
\]
So then 
\[
\sum_{k=0}^\infty (gMg^{-1})^k/k! = \sum_{k=0}^\infty gM^kg^{-1}/k! = g \left( \sum_{k=0}^\infty M^k/k! \right) g^{-1} = g \exp(M) g^{-1}.
\]

\paragraph{Structured Network Output}

So the last layers of the neural network will be $\exp \circ F$ which outputs an invertible $2 \times 2$ matrix and satisfies 
\[
(\exp \circ F)(\rho_{reg}(\phi) f) = \rho_1(\phi) (\exp \circ F)(f) \rho_1(-\phi).
\]

\section{Algorithms}

\subsection{Score Function for Probabilistic Forecast}

We consider a probabilistic approach to trajectory forecasting. Ideally, we want the predicted distribution to be both sharp and valid under the definition in \Eqref{eq:coverage}. To quantify the uncertainty, we review the classic score function Mean Interval Score (MIS) and derive a 2 dimensional extension, Mean Regional Score (MRS).




\paragraph{Mean Interval Score.}
Mean Interval Score (MIS) is a proper scoring rule \cite{properscoring} for interval forecasts that rewards narrower confidence intervals and encourages coverage. 
Specifically, let $Y  \sim {P}_Y$ be a one-dimensional random variable, and an upper bound $u$, a lower bound $l$ be its estimated $\frac{\alpha}{2}$ and $(1-\frac{\alpha}{2})$ quantiles respectively, MIS is defined using samples $y_i \sim {P}_Y$.
\begin{align}
    \mathrm{MIS}(u,l; Y) = \frac{1}{N} \sum_{i=1}^{N}[(u-l) +
     \frac{2}{\alpha}(y_i-u)\mathbb{1}\{y_i>u\}\nonumber 
     + \frac{2}{\alpha}(l-y_i)\mathbb{1}\{y_i<l\} ] 
    \label{eq:mis}
\end{align}
Where the three terms in  \Eqref{eq:mis} accounts for the width of the interval, and how much the sample point exceeds the upper bound and the lower bound, respectively. MIS has the advantage of being easy to compute and does not require the model to be parametric \cite{properscoring}. It can also be used as an objective function for probabilistic forecasting \cite{wu2021quantifying}.

\paragraph{Mean Regional Score.}
For n-D distributions, we introduce a new metric called the Mean Region Score (MRS), which generalizes MIS to higher dimensions. 
Let $Z \sim {P}_Z$ be an n-dimensional continuous random variable. First, we need to generalize  $\alpha$-quantile and define the  confidence region. Let the region bordering $z$ as $R(z)$, the  confidence region of level $\alpha$ 
\begin{align}
R_P(\alpha) =  \inf \lbrace R(z)  | \int_{R(z)} P( z ) dz \geq \alpha \rbrace
\end{align}
is the smallest region whose probability exceeds  level $\alpha$.

\begin{figure}
    \centering
  \includegraphics[width=.8\linewidth]{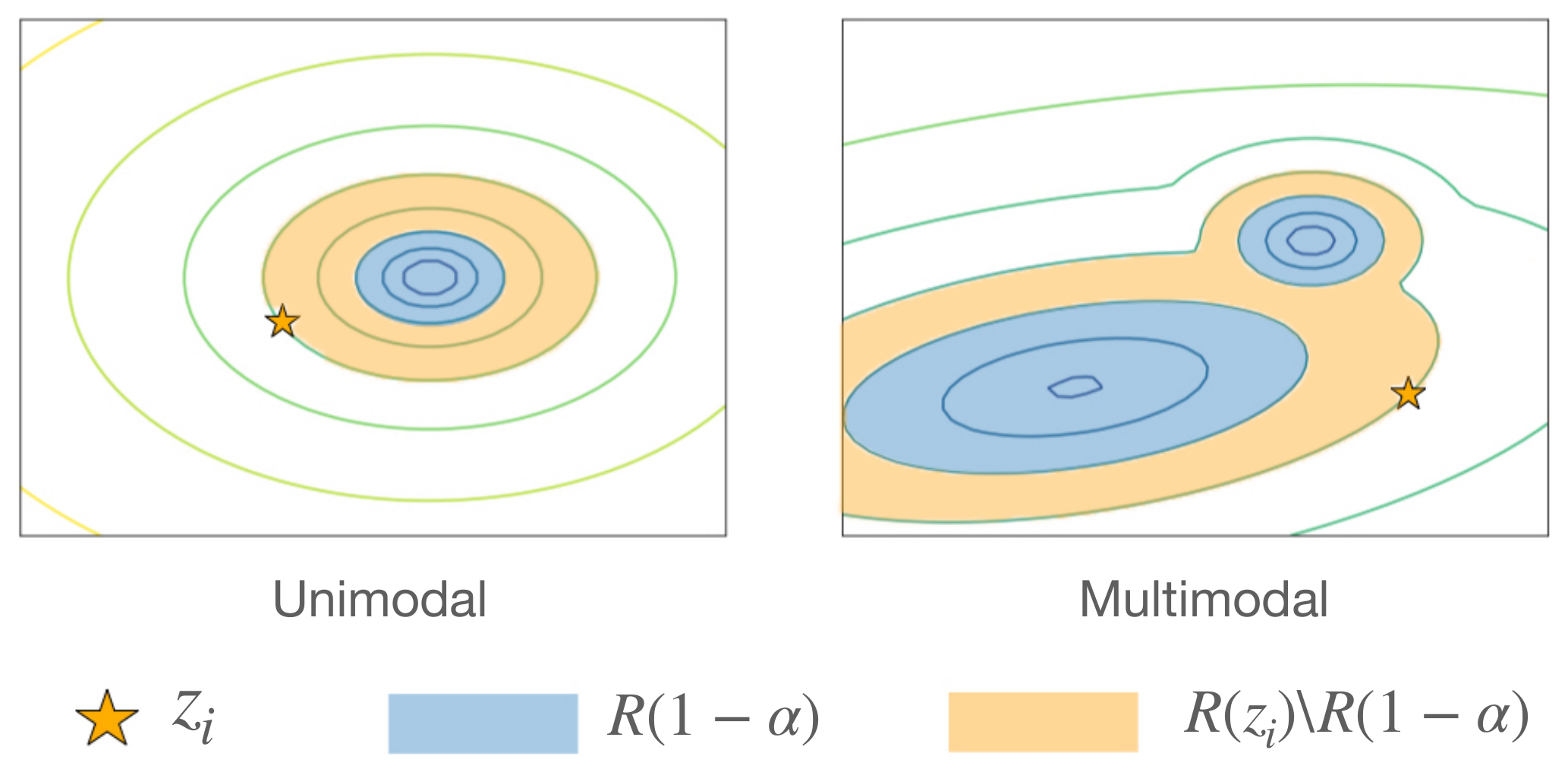}
  \caption{MRS illustration for   unimodal and multimodal distributions. Given a sample $z_i$, MRS  calculates the area of blue region plus that of the the orange region scaled by $1/\alpha$. }
  \label{fig:mrs}
\end{figure}
When the distribution is Gaussian $P_Z =\mathcal{N}(\mu, \Sigma)$ with mean $\mu \in \mathbb{R}^n$ and the covariance matrix $\Sigma \in \mathbb{R}^{2\times 2}$ that is positive definite. The $(1-\alpha)$ confidence region of an n-D multivariate normal distribution $P = \mathcal{N}(\mu, \Sigma)$ is an ellipsoid that can be written as \citep{slotani1964tolerance}:
\[
    R(1-\alpha) = \{ z| (z-\mu)^T \Sigma^{-1} (z-\mu) \leq  \chi_n^2(1-\alpha)\}.
\]
where $\chi_n^2$ is the chi-squared distribution with n degrees of freedom.
Then, for a given sample $z_i \in \mathbb{R}^n$, we can draw an ellipsoid whose edge coincides $z_i$ defined as 
\[
    R(z_i) = \{z|(z-\mu)^T\Sigma^{-1}(z-\mu) \leq c'\} 
\]
where $c' = (z_i - \mu)^T \Sigma^{-1} (z_i - \mu)$.


The MRS score of the $(1-\alpha)$ prediction interval thus can be evaluated as:
\begin{align}
    \mathrm{MRS}(R ;Z) = \frac{1}{N} \sum_{i=1}^N  [|R(1-\alpha)| +   \frac{1}{\alpha}| R(z_i) \backslash R(1-\alpha)| \mathbb{1} \{z_i \in R^C(1-\alpha)\}]
\label{eq:mrs}
\end{align}
where the first term  corresponds to the area of the confidence region and the second term measures how much each data $z_i$ deviates from the region. $|\cdot|$ measures the size of the set.  It is also easy to prove (see Appendix \ref{app:properscoring}) that MRS is a proper scoring rule \cite{properscoring}, which means MRS is optimized if and only if the distribution $P$ recovers that of $Z$.

Figure \ref{fig:mrs} illustrates MRS for both unimodal and multimodal distributions.  Note that when the distribution is single-modal, the confidence interval in \Eqref{eq:mrs} is both \textit{equally tailed} and the \textit{shortest} confidence interval. For  multi-modal distributions,  we solve for the \textit{smallest} confidence region to cover probable events. In appendix \ref{app:mrsgeneral} we present a method to estimate MRS for arbitrary 2-D distributions, where we divide the domain space into small grids to estimate density, and obtain $R(1-\alpha)$ and $R(z_i)$ by assembling high density regions.
%


\subsection{Empirical method for estimating MRS for general 2D distribution }
\label{app:mrsgeneral}

We present an empirical method to estimate MRS given any 2D distribution. We first discretize the 2D space that is the domain of the probability distribution into $n$ grid cells, indexing them as $g_i$ where $i\in \{1,\ldots,j\}$. Then we can numerically estimate the density of each grid cell $p(g_i)$ for $i= 1,\ldots,n$. The confidence region can be estimated by selecting the grid cells with the highest likelihoods until they sum to $1-\alpha$. 
\[
R(1-\alpha) = \bigcup_{i\in K} g_i, \quad K =  \text{argmin}_K: \Sigma_{i\in K} p(g_i) \geq 1-\alpha
\]
where $K$ is the superlevel set of grid cells $g_i$ whose empirical probability density is at least  $1-\alpha$.

Similarly, we define the estimated region 
\[ 
R (z_i) = \bigcup g_i \quad \text{ for } p(g_i) \geq p(z_i)
\]
as the superlevel set of data point $z_i$. This way, we can calculate MRS for any 2D using the estimated $R(1-\alpha)$ and $R(z_i)$ in the MRS formula. 
\[
    \mathrm{MRS}(R ;Z) = \frac{1}{N} \sum_{i=1}^N  [|R(1-\alpha)| +     \frac{1}{\alpha}| R(z_i) \backslash R(1-\alpha)| \mathbb{1} \{z_i \in R^C(1-\alpha)\}]
\]

Figure \ref{fig:empiricalmrs} illustrates the idea, this method is used to estimate the MRS scores for multimodal methods \texttt{Trajectron++} and \texttt{MFP} in our experiments.

\begin{figure}[!h]
  \centering
  \includegraphics[width=.5\linewidth]{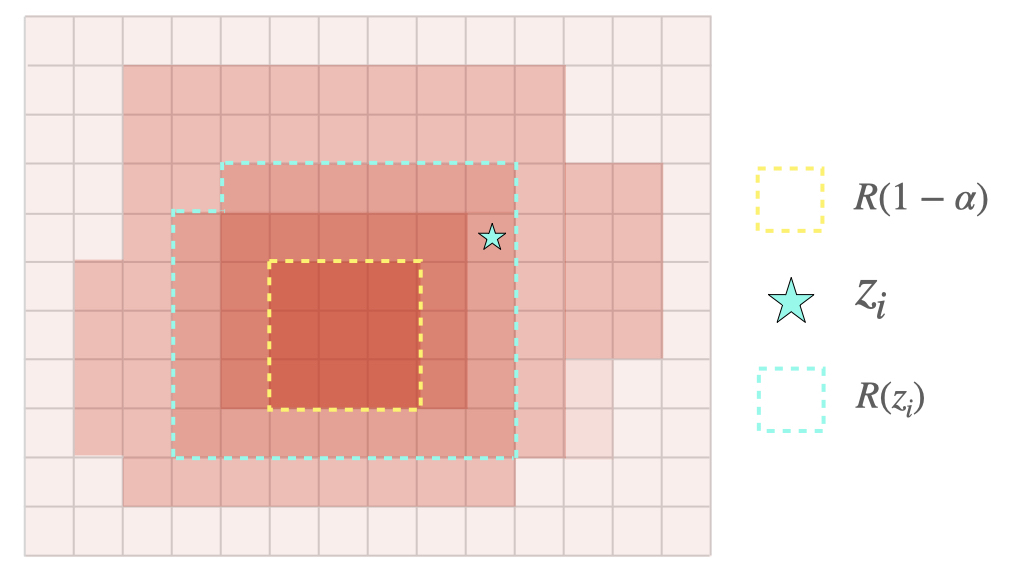}
  \caption{Illustration for empirical MRS estimation, where we calculate the probability density at each grid, and select the highest density grids to construct $R(z_i)$ and $R(1-\alpha)$. This way, we may approximate MRS numerically.}
  \label{fig:empiricalmrs}
\end{figure}

\subsection{Conformal algorithm for 2-D Time Series}
\label{app:conformal}
We extend the inductive conformal prediction (ICP) algorithm \cite{shafer2008tutorial} to conformally estimate a confidence region for 2D forecasts. Given an exchangeable set of trajectory observations $\mathcal{D} = \{ (\mathbf{x}_{1:t}^{(i)}, \mathbf{e}^{(i)}), \mathbf{x}_{t+1:t+k}^{(i)}\}_{i=1}^l$ and a new sample $(\mathbf{x}_{1:t}^{(l+1)}, \mathbf{e}^{(l+1)})$, the ICP algorithm returns $k$ prediction intervals,$[\Gamma^{\alpha}_1, \ldots, \Gamma^{\alpha}_k]$, one for each timestep, such that:
\[
\mathbb{P}[\: \forall j \in \{1,\ldots, k\}, \: x_{t+j} \in \Gamma^{\alpha}_j(\mathbf{x}^{(l+1)})\: ] \geq 1-\alpha
\]
for any underlying predictive model. This inequality is the \textit{validity} condition. \cite{2021conformal} proves the 1-D case for time-series forecasting validity; we leave proving it for the 2-D case for future work.

ICP requires an underlying forecasting model and a choice of nonconformity score. In our case, LSTM and ECCO are both models that take as input $(\mathbf{x}_{1:t}, \mathbf{e})$ and outputs $\mathbf{x}_{t+1:t+k}$. We select root mean square error (RMSE) as our nonconformity score as it is commonly used in the 1-d setting and naturally extends to high dimensions. We split the training set into the proper training set and a calibration set of equal size: $\mathcal{D} =\mathcal{D}_{train} \bigcup \mathcal{D}_{cal}$. We train our model $M$ on $\mathcal{D}_{train}$ and obtain the critical nonconformity scores $\hat{\gamma_1}, \ldots, \hat{\gamma_k}$ as in algorithm 1 in \cite{2021conformal}.
Then, for input $(\mathbf{x}_{1:t}, \mathbf{e})$ and $ \hat{x}_{t+1:t+k} = M(\mathbf{x}_{1:t}, \mathbf{e})$ we can construct a set of confidence intervals 

\[
\Gamma^{\alpha}_j(\mathbf{x}^{(l+1)}_{(1:t)}) = \{ x\in \mathbb{R}^2 \; | \; RMSE(x, \hat{x}_j) \leq \hat{\gamma}_j\} \; \text{ for } j \in\{1,\ldots, k\} 
\]

We illustrate some example scenes with confidence regions provided by conformal-ECCO in figure \ref{fig:conformal}. Note that the regions for timestep $j$ are given by $\hat{\gamma}_j$, hence are the same for every scenario. With our choice of RMSE as the nonconformity score, the region is a circle; they appear elliptical due to scale compression. 

\begin{figure*}[ht]
\centering
\begin{subfigure}
  \centering
  \includegraphics[width=.245\linewidth]{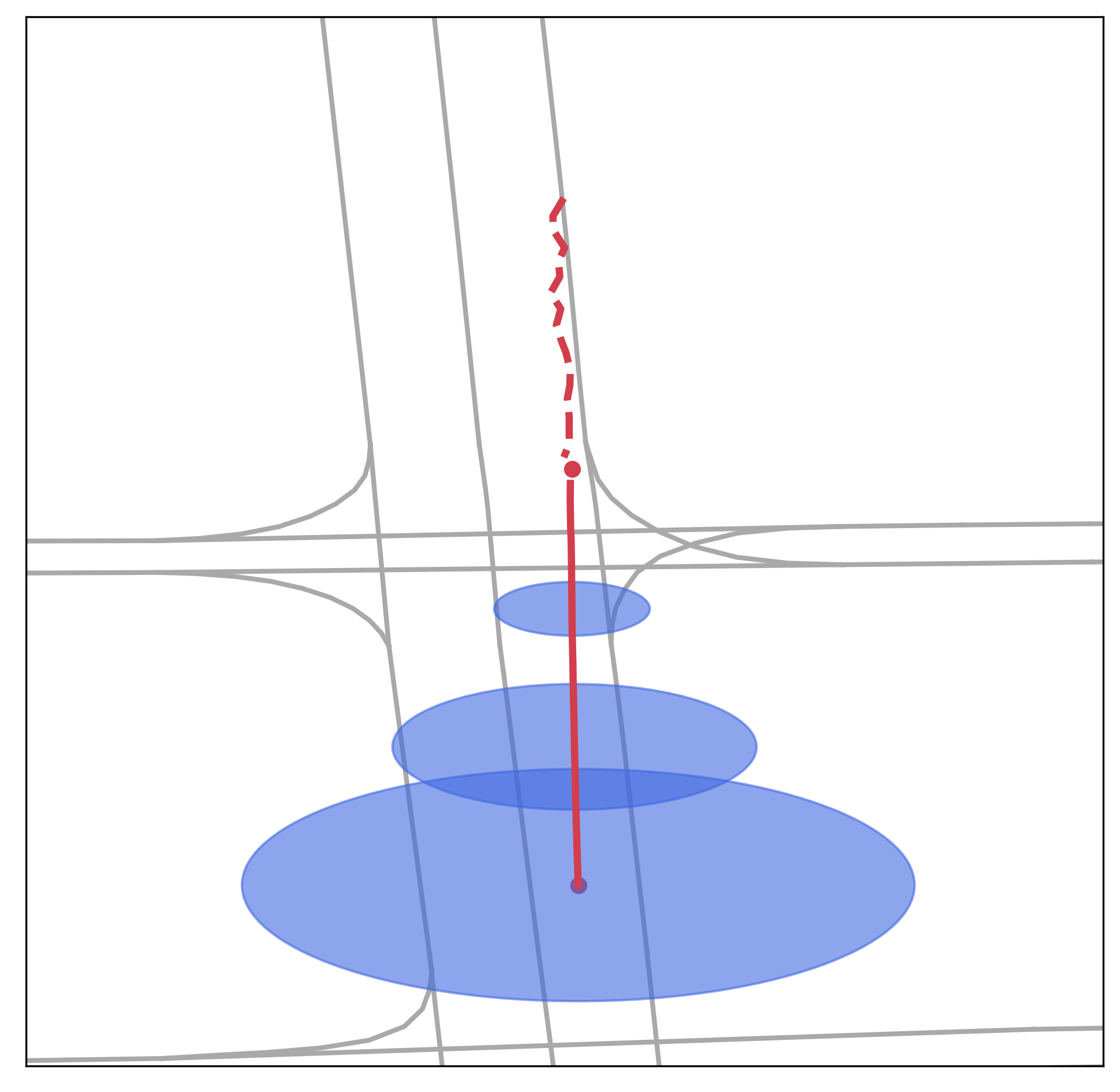}
\end{subfigure}%
\begin{subfigure}
  \centering
  \includegraphics[width=.245\linewidth]{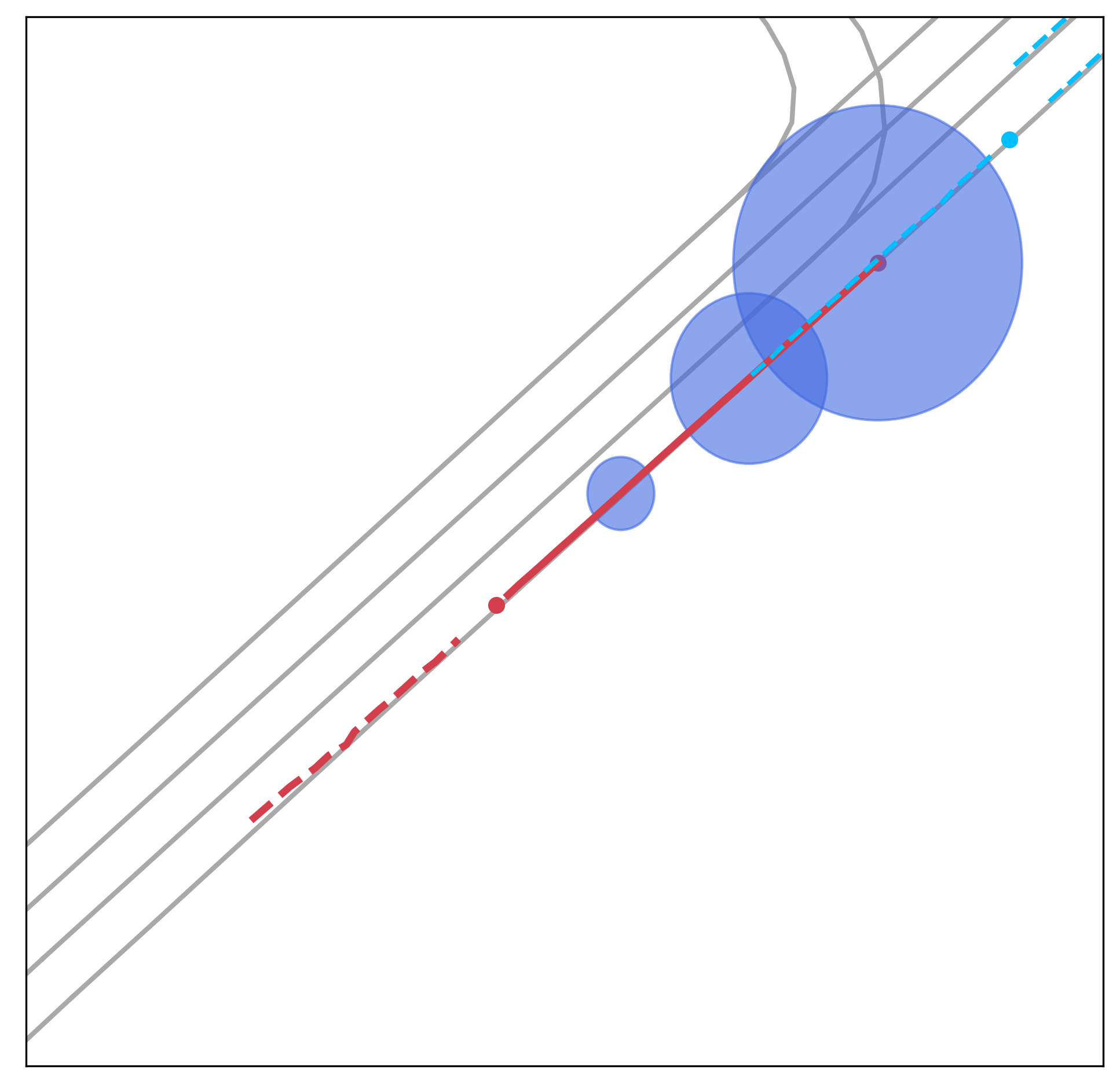}
\end{subfigure}
\begin{subfigure}
  \centering
  \includegraphics[width=.245\linewidth]{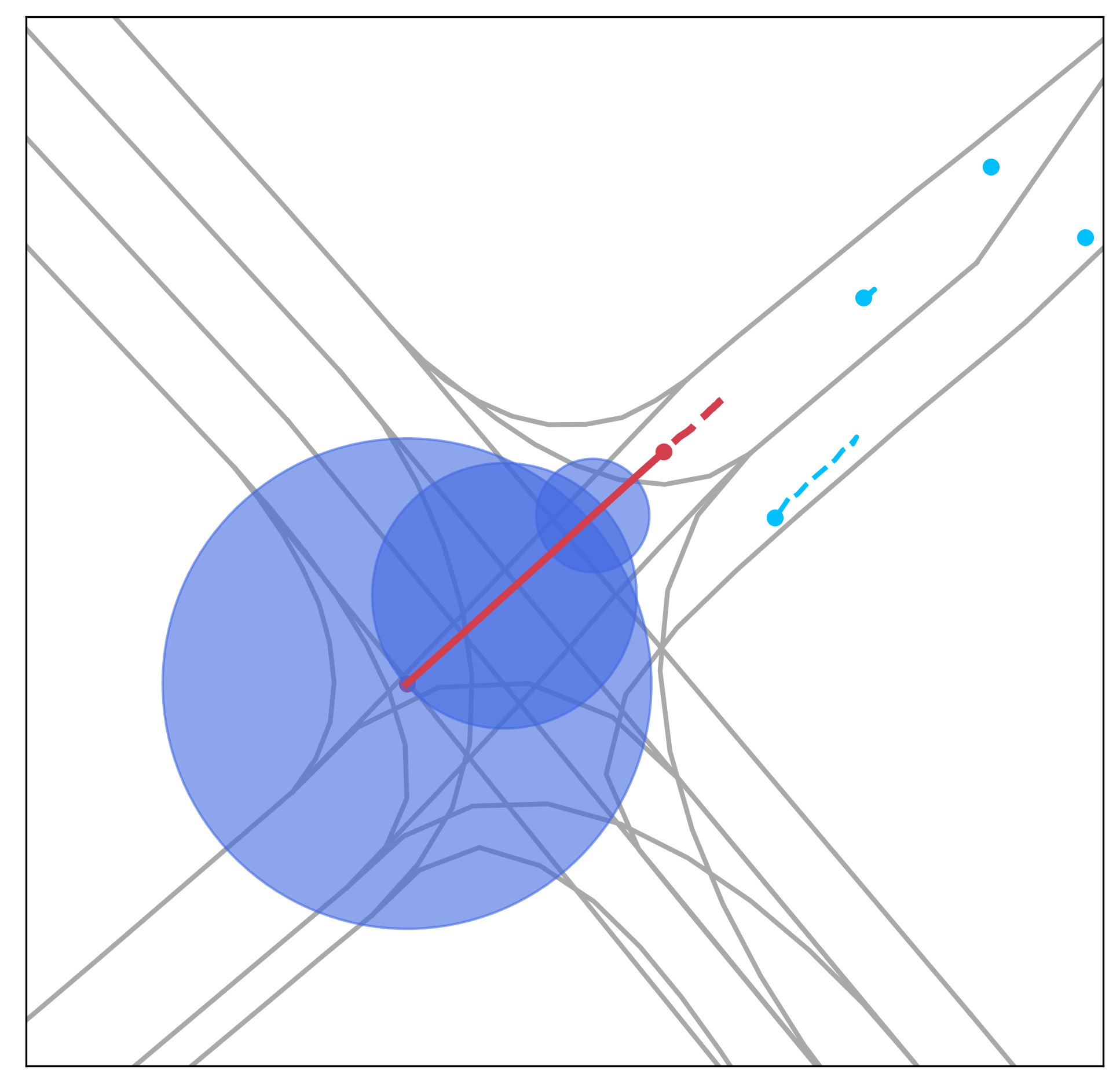}
\end{subfigure}%
\begin{subfigure}
  \centering
  \includegraphics[width=.245\linewidth]{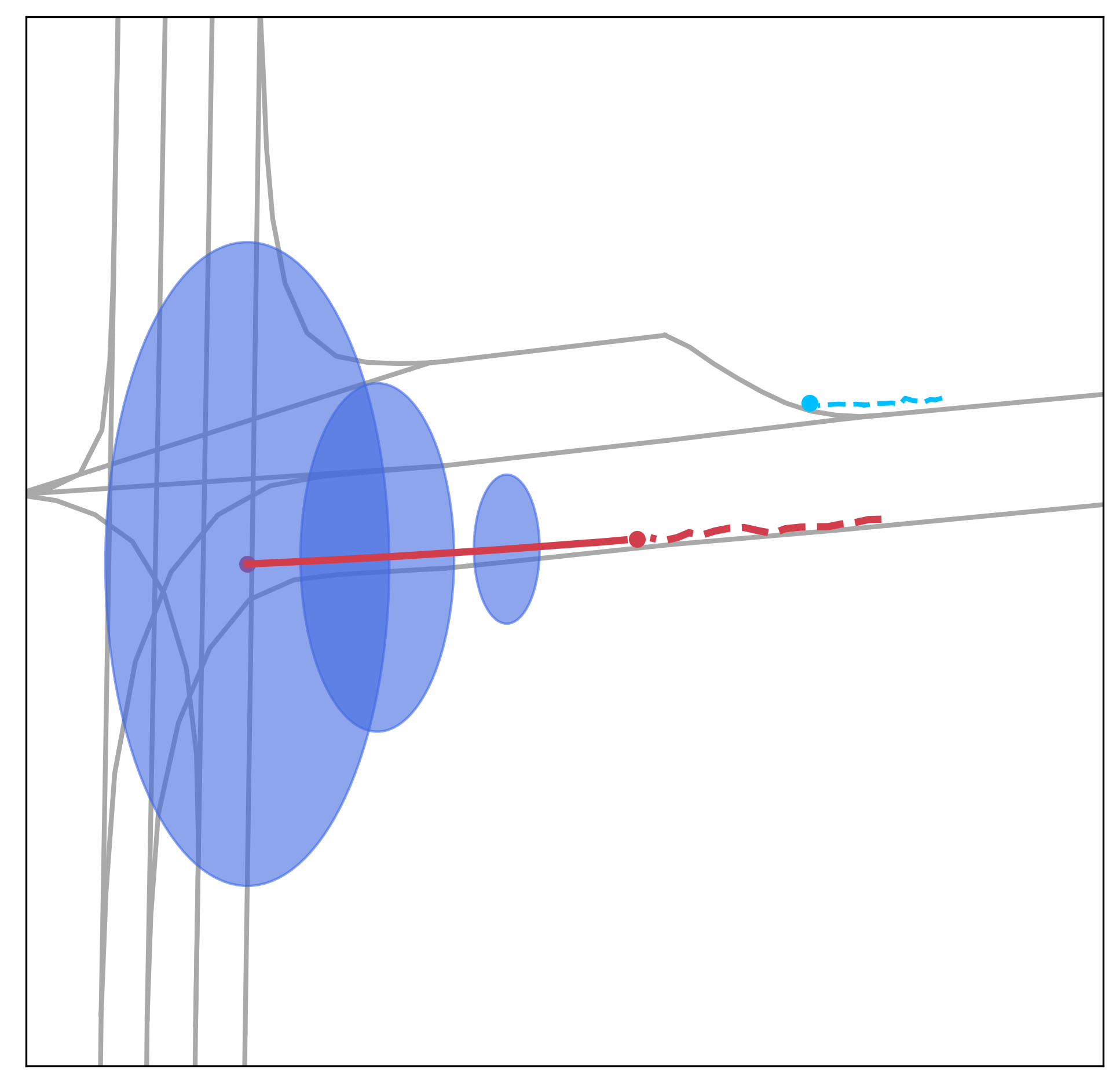}
\end{subfigure}
\caption{Scene illustrations for confidence regions given by conformal-ECCO. Note that in order to achieve 90\% coverage, the region is larger than needed, especially in straight-lane cases like left two scenarios. Due to the lack of underlying assumption about the distribution, the model isn't able to adjust the shape of uncertainty according to data either.}
\label{fig:conformal}
\end{figure*}

\section{Experiment Details}

\subsection{Implementation Details}
\label{app:exp}

The numbers  in \autoref{tab:big-table} are reported on the official validation set for Argoverse, and on a 10\% test split for TrajNet++. For the particle dataset, the reported numbers are also on a 10\% test split.

\paragraph{Dataset and Preprocessing} 
The synthetic particle dataset is generated by the spring dynamics simulator in \cite{nri}. We added $\sigma=0.01$ noise to the dynamics each step to introduce some uncertainty. The dataset consists of 12,000 time series of trajectories of 5 interacting particles over 50 timesteps ; we use 30 steps as input and predict the following 20 steps. 1,000 samples are used for model validation and 1,000 are used for testing.

The Argoverse autonomous vehicle dataset contains 205,942 samples, consisting of diverse driving scenarios from Miami and Pittsburgh.  We split 90/10 into a training set and validation set of size 185,348 and 20,594 respectively. The official validation set of size 39,472 is used for testing and reporting performance. We preprocess the scenes to filter out incomplete trajectories and cap the number of vehicles modeled to 60. If there are less than 60 cars in the scenario, we insert dummy cars into them to achieve consistent car numbers. For map information, we only include center lanes with lane directions as features. Similar to vehicles, we introduce dummy lane nodes into each scene to make lane numbers consistently equal to 650.

We use the latest release of the TrajNet++ dataset (Update 4.0) for our pedestrian experiments. TrajNet++ is a compiled set of pedestrian trajectories captured in both
indoor and outdoor locations such as in universities, hotels, Zara, and train stations. The sample in this dataset is 21 timestamps long, and the goal is to predict the 2D spatial positions for each pedestrian in the future 12 timestamps given the first 9.
The pedestrian dataset contains 240,896 samples, which we split 80/10/10 into train, validation, and test sets. Similar to Argoverse, we filter out incomplete trajectories in processing and either cap or insert dummy pedestrians such that each scene has 60 agents. No map information was used in the pedestrian dataset. 

We include our code for preprocessing, model implementation, and training in the supplementary materials.

\paragraph{Hyperparameters and Training Details} We trained the PECCO model with 4 equivariant continuous convolution layers of hidden size of (8, 16, 16, 16) respectively for Argoverse, and (4, 8, 16, 16) for the Trajnet++ pedestrian dataset. Our models are all trained with Adam optimizer with a base learning rate $r = 0.001$, and linear learning rate scheduler set to $\gamma = 0.95$.  For Argoverse task, we set the CtsConv radius to be 40, and for the pedestrian task we set it to be 6. PECCO's Arogoverse model has 129k parameters. For comparison, Trajectron++ has 127k and MFP has 67K. 

All our models without map information are trained for 10K iterations with batch size 32 with learning rate updating every 150 iterations. Most of our experiments are performed on a server with 4 RTX 2080 Ti GPUs, and it takes around 9-12 hours to finish training. We run each experiment 3 times with different random initialization and data order. The numbers reported in tables \ref{tab:big-table}, \ref{tab:conformal}, and \ref{tab:deterministic} are the mean and standard deviation of those 3 runs. 

\subsection{Deterministic Baseline Results}
\label{app:deterministic}
We present numbers for deterministic baseline models for Argoverse and Trajnet++. By sampling, we achieve better minADE/minFDE performance with probabilistic models, but they serve as valuable baselines to illustrate the difficulty of the task. 
\begin{table*}[h]
\centering
\begin{tabular}{ c| c c | c c  } 
 \toprule
   &  \multicolumn{2}{c}{Argoverse}  & \multicolumn{2}{c}{Trajnet++}  \\ 
  \midrule
  \textbf{Model}  &  ADE$ \downarrow$ & FDE$\downarrow$ & ADE$ \downarrow$ & FDE$\downarrow$ \\ 
  \midrule
Constant Velocity & 2.77 & 6.16 &1.21 & 2.37\\
Nearest Neighbor   & 3.52 & 7.85 & 1.25 & 2.61\\
LSTM    & 1.97 \plusminus{.05} & 4.98 \plusminus{.31} &1.01  \plusminus{.02}  & 1.98  \plusminus{.08} \\
CtsConv & 1.87 \plusminus{.06} & 4.43 \plusminus{.28} &1.35  \plusminus{.05} & 2.97 \plusminus{.16} \\
ECCO    & 1.68 \plusminus{.04} & 3.98 \plusminus{.19}  & 0.94  \plusminus{.01}  & 2.05 \plusminus{.03} \\
\bottomrule
\end{tabular}
\caption{ Deterministic baselines on the Argoverse and TrajNet++ dataset}
\label{tab:deterministic}
\end{table*}

\end{document}